\pdfoutput=1
\documentclass[11pt]{article}

\usepackage[letterpaper,margin=1in]{geometry}
\usepackage{natbib}

\usepackage[utf8]{inputenc}
\usepackage[T1]{fontenc}
\usepackage{lmodern}
\usepackage{url}
\usepackage{microtype}
\usepackage{graphicx}
\usepackage{subcaption}
\usepackage{booktabs}
\usepackage{multirow}

\usepackage{colortbl}
\usepackage{xcolor}
\usepackage{adjustbox, enumitem}

\usepackage{makecell}
\usepackage{array}

\usepackage{hyperref}
\hypersetup{
  colorlinks=true,
  allcolors=blue!60!black,
  pdftitle={SAG-Sep: Sparse Augmented Graphs and Onion-Guided Search for Rounded Capacity Cut Separation},
  pdfauthor={Haoran Liu, Guanyi Wang, Yu Yang}
}

\usepackage{amsmath}
\usepackage{amsfonts}
\usepackage{amssymb}
\usepackage{mathtools}
\usepackage{amsthm}

\usepackage[capitalize,noabbrev]{cleveref}

\theoremstyle{plain}
\newtheorem{theorem}{Theorem}[section]
\newtheorem{proposition}[theorem]{Proposition}
\newtheorem{lemma}[theorem]{Lemma}

\theoremstyle{definition}

\newtheorem{assumption}[theorem]{Assumption}
\theoremstyle{remark}

\title{SAG-Sep: Sparse Augmented Graphs and Onion-Guided Search for Rounded Capacity Cut Separation}

\author{Haoran Liu\\
  Department of Industrial Systems Engineering\\
  University of Florida\\
  Florida, United States\\
  \texttt{haoran.liu@ufl.edu}
  \and
  Guanyi Wang\\
  Department of Industrial Systems Engineering and Management\\
  National University of Singapore\\
  Singapore\\
  \texttt{Guanyi.W@nus.edu.sg}
  \and
  Yu Yang\\
  Department of Industrial Systems Engineering\\
  University of Florida\\
  Florida, United States\\
  \texttt{yu.yang@ise.ufl.edu}
}

\date{}

\begin{document}

\maketitle

\begin{abstract}
  Rounded capacity cuts (RCCs) are among the most effective cuts for capacitated vehicle-routing relaxations, but exact separation is computationally prohibitive at scale. Existing heuristic and neural separators can be fast, but often fail to produce sufficiently many effective cuts within practical separation budgets. We introduce SAG-Sep, a separator that scores node membership in RCC-inducing subsets across vehicle-count levels from a fractional LP solution in a single encoder forward pass, avoiding the iterative predict-and-coarsen inference used by NeuralSEP-style separators. SAG-Sep encodes the LP solution as a Sparse Augmented Graph (SAG), adding flow-informed multi-hop edges that expose long-range routing structure without densifying attention. Beyond LP-aware sparse encoding, we observe that high-value exact RCCs are typically nested, organizing into onion-like subset families, and SAG-Sep's node scores track subset nesting depth. This motivates Onion-Sweep (OS), a training-free decoder that ranks customers by their scores, evaluates the RCC violation of every prefix of the ranking in one vectorized pass, and emits the most violated prefix of each replica together with a cross-replica family of nested proposals, converting independent node scores into structurally coordinated RCC candidates without a tuned threshold or probing band. On the NeuralSEP benchmark with $N=1000$ customer instances, SAG-Sep with OS reduces the average root gap by 17.7\% relative to the strongest NeuralSEP baseline, while the single-pass SAG-Sep separator achieves a roughly 30-fold per-iteration neural separation speedup. Our findings suggest that LP-aware sparse encoding and the nested structure of high-value cuts are structural patterns worth leveraging in neural separation more broadly.
\end{abstract}

\section{Introduction}
\label{sec:intro}

The capacitated vehicle routing problem (CVRP) is a canonical NP-hard combinatorial optimization problem and a central testbed for large-scale
routing algorithms \citep{uchoa2017new}. State-of-the-art exact
solvers are built on branch-and-cut (B\&C) or branch-price-and-cut (BPC)
frameworks \citep{lysgaard2004new,pessoa2020generic,you2025routeopt}, where
cutting planes iteratively tighten linear relaxations to improve dual bounds. Among the many valid inequalities proposed for CVRP, \emph{rounded capacity cuts} (RCCs) are particularly effective and are routinely used in both B\&C and BPC solvers \citep{lysgaard2004new,fukasawa2006robust}.

Despite their effectiveness, RCCs pose a fundamental scalability challenge. Separating violated RCCs is NP-hard \citep{Diarrassouba2017}, and exact separation formulations rapidly become computationally prohibitive as the number of customers grows \citep{Augerat1995,fukasawa2006robust,pavlikov2024exact}. Although exact methods provide strong per-iteration tightening of the linear program (LP) relaxation, they are impractical at large scales, forcing solvers to rely on heuristics. The widely used CVRPSEP heuristic \citep{lysgaard2003cvrpsep} performs well on small-to-medium instances, but on larger problems it often enters a low-throughput regime, generating few effective cuts and delivering only marginal improvements in the dual bound.

Recent work has explored learning-based alternatives. NeuralSEP \citep{kim2024neural} introduced the first neural RCC separator by training a graph neural network (GNN) to predict cut-defining customer subsets from fractional LP solutions. However, operating solely on the sparse LP-support graph limits scalability: fixed-depth message passing captures only local structure, while deeper networks suffer from over-smoothing and over-squashing \citep{li2018deeper,alon2021on,wu2021representing}. NeuralSEP addresses this with recursive graph coarsening, but this multi-stage pipeline incurs substantial training and inference overhead, restricts GPU parallelism, and requires repeated CPU–GPU synchronization. These limitations motivate a separator that captures long-range dependencies without sequential coarsening and instead enables highly parallel, hardware-efficient computation.

To address these challenges, we propose \textbf{SAG-Sep}, a single-pass RCC separator built around a \textbf{Sparse Augmented Graph} (SAG) representation of the current LP solution. SAG starts from the LP-support graph and augments it with a small number of flow-informed multi-hop edges, exposing long-range routing structure while retaining sparse attention. Given this representation, SAG-Sep scores customer membership in RCC-inducing subsets across vehicle-count levels in one encoder pass, avoiding the iterative predict-and-coarsen procedure of NeuralSEP-style separators.

The continuous vehicle-count conditioning also creates a richer test-time search problem. Aggressive probing of the frozen network showed that it contains useful candidate cuts beyond the hard-threshold subset of each replica, but such probing treats the replicas as independent predictions. Studying the exact RCC separator revealed a consistent pattern: high-value exact cuts often form nested \emph{onion} families, providing cut diversity through ordered subset growth rather than unrelated threshold decisions. This observation leads to \textbf{Onion-Sweep (OS)}, a training-free decoder that reads subsets directly off the score ranking instead of thresholding it: within each replica, OS evaluates the RCC violation of every prefix of the score ranking in one vectorized pass and emits the most violated one; across replicas, it aggregates the scores into a global priority ordering and emits the nested prefixes at which the required vehicle count increases. No tuned threshold, band width, or sampling is involved, and the trained encoder is unchanged.

\subsection{Contributions}
We summarize our contributions as follows:
\begin{itemize}
\item We introduce \textbf{SAG-Sep}, a one-shot RCC separator based on a \textbf{Sparse Augmented Graph} encoding of the fractional LP solution. The SAG representation preserves the sparsity of the LP-support graph while adding flow-informed multi-hop edges, enabling the model to capture long-range routing dependencies without dense attention or NeuralSEP-style iterative coarsening. This design substantially improves GPU efficiency and sharply reduces neural separation time while retaining strong root-node bounds.

\item We study test-time enhancement in the learned probability space: aggressive conditioned probing, which densifies the vehicle-count conditioning grid, improves root bounds and reveals that the trained network contains useful cuts beyond those exposed by hard-thresholded replicas.

\item We further analyze the learned score space together with exact RCC separation and identify a stable nested, onion-like structure among high-value exact cuts. This motivates \textbf{Onion-Sweep (OS)}, a training-free decoder that replaces hard thresholding by a violation-guided sweep over the score ranking within each replica and by onion-guided nested proposals across replicas, removing the threshold dependence of independently thresholded outputs and substantially improving separation effectiveness.

\item More broadly, the design illustrates how a neural module can take over an NP-hard separation subroutine at the root node of an exact solver, accelerating relaxation tightening without sacrificing the validity of generated cuts.
\end{itemize}

\subsection{Related Work}
\label{sec:related}

\paragraph{End-to-end learning for CVRP.}
A substantial body of work applies machine learning to CVRP to produce high-quality \emph{primal} solutions: neural construction heuristics formulated routing as sequence generation with pointer networks \citep{vinyals2015pointer}, strengthened by reinforcement learning \citep{Nazari2018}, attention-based architectures \citep{Kool2019}, and symmetry- and multi-start-based training such as POMO \citep{Kwon2020pomo} and Sym-NCO \citep{kim2022sym}.

Recent methods scale to very large instances through decomposition or divide-and-conquer \citep{li2021learning,ye2024glop,zhengUDCUnifiedNeural}, or \emph{improve} incumbents through learned local search \citep{lu2019learning,xin2021neurolkh,ma2023learning}. Although often highly competitive on large benchmarks, these neural heuristics operate purely in the primal space and typically provide no guarantees on solution quality.

\paragraph{Learning in exact solvers.}
A complementary direction incorporates learning directly into \emph{exact} solvers, leveraging solver states such as LP solutions and dual prices within B\&C or BPC frameworks while preserving optimality guarantees \citep{bengio2021machine}. A major line of work focuses on \emph{learning to branch} \citep{khalil2016learning, alvarez2017machine, gasse2019exact}, with extensions to more general branching schemes \citep{yang2021multivariable, yang2022learning} and multi-stage procedures \citep{you2023two}. Learning has also been explored for \emph{column generation}, where models propose or prioritize promising columns to reduce pricing effort \citep{zhang2022learning, Morabit2024}, prefill the master with learned column pools \citep{liu2026imperfect}, or stabilize the dual iterates with learned dual-optimal inequalities \citep{you2026learned}.

In contrast, \emph{learning for cutting planes} remains underexplored: most work addresses \emph{cut selection} from large candidate pools \citep{tang2020reinforcement, huang2022learning}, whereas learning-based \emph{cut generation}, directly proposing new valid inequalities, is much less studied, particularly when exact separation is intractable. We target this gap by predicting RCC-defining subsets directly from fractional LP solutions, extending recent neural separation approaches \citep{kim2024neural}.

\section{Preliminaries}
\label{sec:prelim}

\subsection{The Capacitated Vehicle Routing Problem (CVRP)}
\label{sec:prelim-cvrp}

Let $G=(V,E)$ be a complete undirected graph with depot $0$ and customer set
$V^+ := V\setminus\{0\}$. Each customer $i\in V^+$ has integer demand $d_i\in\mathbb{Z}_{>0}$ and all
vehicles have integer capacity $C\in\mathbb{Z}_{>0}$, with $d_i\le C$ and $d_0=0$. Let $c_{ij}\ge 0$ denote the travel cost of edge
$\{i,j\}\in E$. The CVRP seeks a set of depot-rooted routes that serve every customer exactly once, incurring the minimum possible total travel cost.

A classical two-index formulation uses edge variables $x_{ij}$, with $x_{ij}\in\{0,1\}$
for customer--customer edges and $x_{0i}\in\{0,1,2\}$ for depot
edges (allowing two traversals in undirected formulations). The degree constraints $\sum_{j\in V\setminus\{i\}} x_{ij} = 2,\ \forall i\in V^+$ ensure that each customer has exactly two incident edges, and hence is visited exactly once. The objective is $\min \sum_{\{i,j\}\in E} c_{ij}\, x_{ij}$.  To tighten the relaxation and eliminate infeasible subtours, one can iteratively add capacity inequalities. For any subset $S\subseteq V^+$, define the cutset $\delta(S) := \{\{i,j\}\in E : i\in S,\, j\in V\setminus S\}$ and its load $x(\delta(S)) := \sum_{\{i,j\}\in \delta(S)} x_{ij}$. The (strong) capacity
cut is
\[
  x(\delta(S)) \;\ge\; 2\, r(S),
  \qquad \forall S\subseteq V^+,\, |S|\ge 2,
\]
where $r(S)$ denotes the minimum number of vehicles required to serve the
customers in $S$.

Since computing $r(S)$ is NP-hard, solvers typically use the rounded lower bound $r(S)\ge \left\lceil d(S)/C \right\rceil$, where $d(S):=\sum_{i\in S} d_i$, yielding the \emph{rounded capacity cut} (RCC):
\begin{equation*}
  x(\delta(S)) \;\ge\; 2\left\lceil \frac{d(S)}{C} \right\rceil,
  \qquad \forall S\subseteq V^+,\, |S|\ge 2.
\end{equation*}

Because RCCs form an exponential family, solvers generate them on demand: they repeatedly solve the current LP relaxation and invoke separation routines to detect violated RCCs, which are added until the relaxation is sufficiently tight or further bound improvements become marginal.

\subsection{Exact Separation of RCCs}
\label{sec:prelim-rccsep}

Given a fractional LP solution $\bar x\in\mathbb{R}_+^E$, RCC separation aims to identify a customer subset $S\subseteq V^+$ such that the capacity inequality is violated, i.e., $\bar x(\delta(S)) < 2\left\lceil d(S)/C \right\rceil$. Exact separation seeks the subset with the maximum violation.

Following the classical formulation of \citet{fukasawa2006robust}, we enumerate candidate right-hand sides by fixing the number of required vehicles. For each $m\in\{0,\dots,\lceil d(V^+)/C \rceil - 1\}$, we solve a mixed-integer program that searches for a subset $S$ whose total demand exceeds $mC$ (thus requiring at least $m{+}1$ vehicles) while minimizing the crossing value $\bar x(\delta(S))$. Intuitively, this formulation looks for a ``light'' cut in the fractional solution that separates a demand-heavy customer group.

Let $y_i\in\{0,1\}$ indicate whether customer $i$ belongs to $S$ ($y_i=1$ iff $i\in S$), and let $\mu_{ij}\ge 0$ indicate whether edge $\{i,j\}$ crosses the induced cut. The separation problem is:
\begin{align*}
\min_{y,\mu} \;\; & z(m) \;=\; \sum_{\{i,j\}\in E} \bar x_{ij}\, \mu_{ij}  \\
\text{s.t.}\;\; & \mu_{ij} \;\ge\; y_i - y_j, && \forall \{i,j\}\in E, \\
& \mu_{ij} \;\ge\; y_j - y_i, && \forall \{i,j\}\in E, \\
& mC + 1 \;\le\; \sum_{i\in V^+} d_i\, y_i, \;\, &&  \\
& y_0 \;=\; 0, \\
& y_i \in \{0,1\}, && \forall i\in V^+, \\
& \mu_{ij} \;\ge\; 0, && \forall \{i,j\}\in E.
\end{align*}
Let $(y^*,\mu^*)$ be an optimal solution and define $S(m)=\{i\in V^+ : y_i^*=1\}$. The corresponding RCC has a right-hand side of at least $2(m+1)$ and leads to a violation whenever $z(m) < 2(m+1)$. This violated inequality is added to the LP relaxation. We further discuss this exact separation formulation and its variants in Appendix~\ref{sec:two_exact_sep}.

\section{Algorithmic Design of SAG-Sep}

\subsection{Pipeline Overview}
\label{sec:pipeline}

Figure~\ref{fig:pipeline} summarizes the cut-generation pipeline of \textbf{SAG-Sep}
within a single cutting-plane iteration. Given the current fractional LP solution
$\bar x$, we first construct the LP support graph $G_{\mathrm{LP}}$ and the
corresponding sparse augmented graph $\widetilde G$ as defined in
\Cref{sec:aug-lp-graph}. The pipeline further integrates with \textbf{Onion-Sweep (OS)},
a training-free decoder introduced in \Cref{sec:aggressive-sampling}.

Overall, we use graph replication to produce $J$ augmented graphs $\left\{\widetilde G^{(j)}\right\}_{j=1}^J$ from $\widetilde G$, and process this batch in a single encoder forward pass. The model returns node probabilities $\left\{p_i^{(j)}\right\}_{i\in V^+}$, where $p_i^{(j)}$ estimates whether customer $i$ belongs to an RCC-defining subset. Candidate subsets $S\subseteq V^+$ are then constructed either by simple thresholding or by the more structured OS decoder.

Each candidate subset is translated into a valid RCC by computing the rounded bound $\left\lceil d(S)/C \right\rceil$ and evaluating its violation under $\bar x$, i.e., the difference $2\left\lceil d(S)/C \right\rceil- \bar x(\delta(S))$. Any violated RCC is added to the LP relaxation, which is then re-optimized to obtain an updated solution $\bar x$. This \emph{separate–add–resolve} loop continues until either (i) no violated RCC is identified, or (ii) a prescribed time or iteration budget is exhausted.

\begin{figure}[htbp]
  \centering
  \includegraphics[width=0.8\linewidth]{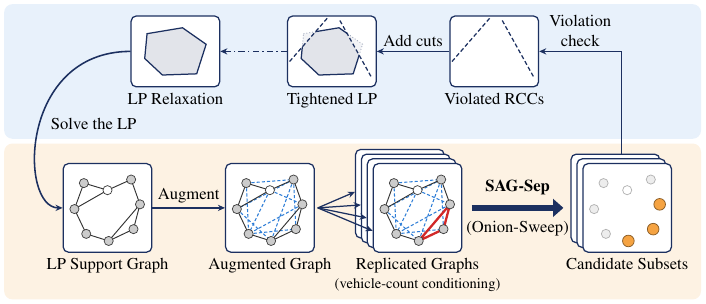}
  \caption{Pipeline overview of SAG-Sep for RCC separation. The current LP solution is encoded as a sparse augmented graph, replicated and evaluated in one batched forward pass, and converted into candidate subsets by thresholding or Onion-Sweep. Violated RCCs are added back to the LP relaxation in a separate--add--resolve loop.}
  \label{fig:pipeline}
\end{figure}

\subsection{Augmented LP Support Graph}
\label{sec:aug-lp-graph}
We first construct a sparse \emph{LP support graph} and then augment it with a small set of informative multi-hop edges, yielding an augmented graph $\widetilde G$ on which SAG-Sep performs masked self-attention.

\paragraph{LP support graph.}
Following the same notation, we define the LP support graph
\[
G_{\mathrm{LP}}=(V,E_{\mathrm{LP}}), \qquad
E_{\mathrm{LP}} := \{\{i,j\}\in E : \bar x_{ij} \ge \tau_1\},
\]
where $\tau_1>0$ is a small threshold. This filtering step retains only edges that carry nontrivial LP mass, resulting in a sparse graph that reflects the dominant structure of the current LP relaxation.

\paragraph{$K$-hop reachability.}
For an integer $K\ge 2$, let $\mathrm{dist}_{\mathrm{LP}}(i,j)$ denote the shortest-path
distance (i.e., the number of edges) from $i$ to $j$ in $G_{\mathrm{LP}}$, with $+\infty$ if $i$ and $j$ are not connected. Define the $K$-hop neighborhood as
\[
\mathcal{N}_K(i) := \{ j\in V\setminus\{i\}:~ \mathrm{dist}_{\mathrm{LP}}(i,j)\le K \}.
\]

\paragraph{Multi-hop LP connectivity score.}
For any $j\in\mathcal{N}_K(i)$, define the connectivity score as
\[
e^{\mathrm{flow}}_{ij}
:= \max_{p:i\leadsto j,\ |p| = \mathrm{dist}_{\mathrm{LP}}(i,j) \le K}\ \min_{(u,v)\in p}\ \bar x_{uv},
\]
where the maximum is taken over all shortest paths $p$ connecting $i$ with $j$ of length at most $K$. This is the largest bottleneck value among the shortest LP-support paths between the two nodes, a hop-limited variant of the classical \emph{widest-path} value. By construction, $e^{\mathrm{flow}}_{ij}=\bar x_{ij}$ when $\{i,j\}\in E_{\mathrm{LP}}$ (via a 1-hop path), while larger values of $e^{\mathrm{flow}}_{ij}$ indicate that $i$ and $j$ are linked through a short path whose edges all take large LP values in $\bar x$.

\paragraph{Augmented graph construction.}
We define the augmented edge set as
\[
  \widetilde E \;:=\; E_{\mathrm{LP}} \;\cup\; E_{\mathrm{syn}},
\]
where $E_{\mathrm{syn}}$ contains \emph{synthetic} edges that connect multi-hop neighbors:
\[
  E_{\mathrm{syn}}
  :=
  \{\{i,j\}: i,j\in V,\ j\in\mathcal{N}_K(i),\ \{i,j\}\notin E_{\mathrm{LP}} \}.
\]
$G_{\mathrm{LP}}$ is typically sparse; with small $K$, the edge growth under the $K$-hop closure remains controlled.

\begin{proposition}
\label{prop:ER_blowup}
Let $N=|V|$. Assume the undirected LP-support graph $G_{\mathrm{LP}}=(V,E_{\mathrm{LP}})$ is sampled from the Erdős-Rényi model $\mathcal{G}(N,\rho)$.
Let $\widetilde G=(V,\widetilde E)$ denote the $K$-hop closure of $G_{\mathrm{LP}}$.
For fixed $K\ge 1$, let $\lambda_N=\rho(N-1)$ and assume $\lambda_N^K=o(N)$. Then
\begin{equation*}
    \frac{\mathbb{E}[|\widetilde E|]}{\mathbb{E}[|E_{\mathrm{LP}}|]}
    =(1+o(1))\sum_{k=0}^{K-1}\lambda_N^k.
\end{equation*}
In particular, if $\lambda_N\to\lambda$, the ratio converges to $\sum_{k=0}^{K-1}\lambda^k$.
\end{proposition}

The proposition (proof in Appendix~\ref{app:proof-er-blowup}) quantifies why the hop radius must stay small: under this model the expected edge count grows geometrically in $K$ through the mean LP-support degree $\lambda$, so a radius of two keeps the expected $|\widetilde E|$, which dominates memory and attention cost, within a factor of about $1+\lambda$ of $|E_{\mathrm{LP}}|$.

Each edge $\{i,j\}\in\widetilde E$ is assigned an input feature vector
\[
  e^{\mathrm{input}}_{ij}
  :=
  \big[\, e^{\mathrm{flow}}_{ij},\ \mathrm{dist}_{\mathrm{LP}}(i,j)\, \big].
\]
The depot is also joined to every customer; missing depot edges get the feature vector $(0,0)$.
Each node $i \in V$ is assigned a feature vector following NeuralSEP~\citep{kim2024neural}:
\[
h_i^{\text{input}} = \left[\, \frac{d_i}{C}, \frac{m}{\max(M-1,1)} \,\right],
\]
where $M=\left\lceil d(V^+)/C \right\rceil$ is the capacity-based lower bound on the number of vehicles required to serve all customers.
To encode vehicle-index information, we create $J=M$ replicas of $\widetilde G$ indexed by $m\in\{0,1,\dots,M-1\}$.
Within replica $m$, all nodes share the same graph-level conditioning value $m/\max(M-1,1)$, so the conditioning values span $[0,1]$.
In contrast, the demand feature $d_i/C$ and all edge features remain unchanged across the $M$ replicas originating from the same $\widetilde G$.

\paragraph{Shared edge encoding.}
The replicated graphs have identical edge structure and edge features and differ only in their conditioning value. At large scales $\widetilde E$ is often much larger than $V$, so duplicating it across all $M$ replicas creates substantial edge-encoding and attention-memory pressure. We therefore encode edge features once, precompute the edge attention bias once per layer, and stack only the node features across replicas, which reduces the edge-side memory from $O(M|\widetilde E|D)$ to $O(|\widetilde E|D)$ for hidden dimension $D$ compared with materializing $M$ graph copies; the replica-dependent attention weights still take $O(M|\widetilde E|H)$ memory, with $H\ll D$ heads. The two effects are complementary: Proposition~\ref{prop:ER_blowup} controls $|\widetilde E|$ itself, while the shared encoding removes the factor $M$ (details in Appendix~\ref{sec:shared_edge_encoding}).

\subsection{Onion-Sweep Decoding}
\label{sec:aggressive-sampling}

\paragraph{From thresholds to violation-guided decoding.}
The basic readout turns the probabilities of replica $j$ into one subset by hard thresholding, $\hat S^{(j)}=\{i\in V^+: p_i^{(j)}\ge\tau\}$, which discards most of the information in the scores: whether a subset induces a violated RCC is decided by the LP solution, not by a fixed cutoff, and a few nodes near $\tau$ can flip the induced cut. \textbf{Onion-Sweep (OS)} instead treats each replica's score ranking as an ordered family of nested subsets and selects among them by their actual violation, through a per-replica \emph{sweep} and a cross-replica \emph{onion} step; both reuse the encoder probabilities and require no additional forward pass.

\begin{center}
\fbox{\begin{minipage}{0.96\linewidth}
\small
\textbf{Onion-Sweep candidate generation.}
Input: probabilities $\{p^{(j)}\}_{j=1}^J$, demands $d$, capacity $C$, current LP solution $\bar x$, and the threshold $\tau$ (used only as a fallback).
\begin{enumerate}[leftmargin=1.5em,itemsep=1pt,topsep=2pt]
\item Initialize candidate pool $\mathcal C\leftarrow\emptyset$.
\item \textbf{Per-replica sweep.} For each replica $j$, sort customers by decreasing $p_i^{(j)}$ into $v_1,\ldots,v_{|V^+|}$ and consider the nested prefixes $P_r=\{v_1,\ldots,v_r\}$. Compute the violation $\nu(P_r)=2\lceil d(P_r)/C\rceil-\bar x(\delta(P_r))$ of every prefix in one pass, using $\bar x(\delta(P_r))=\bar x(\delta(P_{r-1}))+\deg(v_r)-2\sum_{t<r}\bar x_{v_tv_r}$ with $\deg(i)=\sum_{k\in V}\bar x_{ik}$, and accumulating $d(P_r)$ likewise. Let $r^\star$ be the first maximizer of $\nu(P_r)$ over $r\ge 2$. If $\nu(P_{r^\star})>\varepsilon$, add $P_{r^\star}$ to $\mathcal C$; otherwise add the hard-threshold subset $\hat S^{(j)}$.
\item \textbf{Cross-replica onion.} Compute the soft frequency score $a_i=\frac{1}{J}\sum_{j=1}^J p_i^{(j)}$, sort customers as $v_1,\ldots,v_{|V^+|}$ by decreasing $a_i$, form nested prefixes $P_k=\{v_1,\ldots,v_k\}$, and add $P_k$ to $\mathcal C$ whenever $r_k=\lceil d(P_k)/C\rceil$ reaches a new vehicle-count level.
\item \textbf{Finalize subsets.} Remove duplicates and subsets whose RCC is already in the LP, and keep a proposed subset $S$ only if $2\lceil d(S)/C\rceil-\bar x(\delta(S))>0$.
\end{enumerate}
Output: violated RCC-defining subsets in $\mathcal C$.
\end{minipage}}
\end{center}

\paragraph{Violation-guided sweep.}
The sweep replaces the single cutoff $\tau$ by an exhaustive, deterministic evaluation of the $|V^+|$ nested prefixes induced by the ranking. Because $\bar x(\delta(P_r))$ and $d(P_r)$ are updated incrementally, all prefix violations of a replica are obtained in one vectorized pass over the LP-support adjacency, at a cost that is negligible compared with the subsequent LP re-optimization. The threshold subset $\hat S^{(j)}$ is itself a prefix of the same ranking, so the sweep never returns a less violated subset than the threshold readout, and the fallback fires only when no prefix of size at least two is violated; the fallback subset is then itself a non-violated prefix and is removed by the final filter, so up to the tolerance $\varepsilon$ the output of OS does not depend on $\tau$. OS has no band width, probing strategy, or random component; the tolerance $\varepsilon=10^{-6}$ and the minimum subset size of two are fixed throughout.

\paragraph{Onion-guided nested proposals.}
The sweep still treats each replica independently. Exact RCC separation suggests a more structured view: when exact cuts are ordered by their required vehicle count, high-value subsets often form nested \emph{onion} families (the banded exact separator of Appendix~\ref{sec:two_exact_sep} reflects this one-cut-per-demand-band structure, and Appendix~\ref{subsec:onion_structure_analysis} quantifies the nesting of exact labels and its alignment with the encoder's scores). The onion step uses the cross-replica soft frequency score as a global priority ordering, so high-priority customers enter many nested candidates while lower-priority customers enter only in larger prefixes, converting independent node scores into a coordinated family of nested RCC proposals, one per vehicle-count level; onion prefixes with violation below $10^{-4}$ are discarded before the final filter.

\paragraph{Relation to margin-band probing.}
Our earlier decoder, Onion-MBP (Appendix~\ref{sec:hyperparameter_selection}), kept the hard threshold and explored a margin band of width $\Delta$ around it with four fixed probes, combined with the same onion step. OS removes the dependence on $\tau$ and $\Delta$: instead of perturbing nodes near $\tau$, it evaluates the exact violation of every prefix of the ranking; since two of the probes are not prefixes, OS does not contain Onion-MBP as a special case. Onion-MBP is kept as an ablation in Section~\ref{sec:test_time_refinement_results}.

\section{Experiments}
In this section, we organize the empirical evaluation around three main claims:
\begin{enumerate}[label=(\roman*)]
\itemsep 0em
\item \textbf{Single-pass competitiveness.} SAG-Sep matches, and can even outperform, recursive coarsening pipelines while using a deterministic single-pass separator.

\item \textbf{Sparse augmented graph design.} Structured graph augmentation components drive the separation improvement.

\item \textbf{Structured test-time decoding.} Onion-Sweep further strengthens SAG-Sep by exploiting the violation profile of the score ranking and cross-replica nested structure.
\end{enumerate}

\subsection{Experimental Setup}
\label{sec:exp}
We evaluate SAG-Sep against representative exact, heuristic, and learning-based RCC separators. The exact baseline is the classical exact RCC separator of \citet{fukasawa2006robust}; CVRPSEP \citep{lysgaard2004new,you2025routeopt} is run with only its RCC routines enabled; NeuralSEP$_1$ \citep{kim2024neural} is the original recursive-coarsening neural separator, while NeuralSEP$_2$ \citep{sim2025test} is its upgraded pipeline. For the NeuralSEP family, our main comparison focuses on NeuralSEP$_2$ and its stochastic coarsening policy $\pi$ and GraphChip (GC) variants (details in Appendix~\ref{sec:moredisc}; exact separation in Appendix~\ref{sec:two_exact_sep}).
All methods are evaluated at the root node of the cutting-plane procedure on the NeuralSEP benchmark testing instances, using the CVRP formulation described in \Cref{sec:prelim-cvrp}.

All main-table results use a single AMD EPYC 9655P CPU thread for LP solving, a single NVIDIA L4 GPU for neural inference, and a 3600-second root-relaxation time limit, checked between separation rounds (so times can exceed 3600 s); no main result uses multi-GPU inference or multi-threaded LP solving. Code, the trained checkpoint, and the benchmark instances are available at \url{https://github.com/henryleo180/SAG-Sep}. Training details are provided in Appendix~\ref{sec:training_detail}; additional ablations over hop radius, threshold choices, numerical precision, and test-time settings are in Appendix~\ref{sec:moredisc}.

\subsection{Single-Pass Separation vs.\ Recursive Coarsening}
\label{sec:single_vs_recursive}

Table~\ref{tab:cvrp_3panel_iters_with_exact} reports the main root-node comparison: the mean optimality gap $100\,(U-L)/U$ of the final root bound $L$ to the best-known solution value $U$ (BKS), and the total root-node runtime in seconds under the single-thread LP and single-GPU setting.

On small and medium instances, CVRPSEP remains hard to beat: it encodes decades of CVRP-specific separation heuristics in a tightly engineered package. Recursive NeuralSEP pipelines also obtain good gaps, but the improvement over SAG-Sep is modest relative to their runtime cost. At $N=200$, for example, NeuralSEP$_2$ improves the gap by $0.79$ percentage points (pp) over SAG-Sep but is about $7.9\times$ slower; the GC variant improves the gap by $1.41$ pp but is about $12.3\times$ slower.

As the instance size grows, this tradeoff turns in favor of SAG-Sep. At $N=300$ it is faster than both recursive neural pipelines and tighter than NeuralSEP$_2$. From $N=400$ onward, SAG-Sep gives the best gaps in Table~\ref{tab:cvrp_3panel_iters_with_exact}; at $N=400$--$500$ it also has the lowest runtime, while at $N\ge 750$ every method hits the time limit. At $N=1000$, it reduces the gap from $25.82\%$ for $\pi$-NeuralSEP$_2$+GC and $28.72\%$ for CVRPSEP to $22.03\%$, with a per-round neural separation time of only $0.65$s---versus $16$--$19$s for the NeuralSEP pipelines (Appendix~\ref{sec:n1000_exact_compare}). Appendix~\ref{sec:neuralsep1_repro} further shows that SAG-Sep is much stronger than the one-shot NeuralSEP base model, indicating that the large-scale gain comes from the sparse augmented encoding and the cleaner single-pass structure, not only from a speed advantage.

SAG-Sep is not meant to replace CVRPSEP's carefully engineered small-instance heuristics; its value is a stronger learned encoding and a high-throughput separation path, which on large instances translate into better bounds under the same root limit and provide a cleaner platform for Onion-Sweep.

\begin{table*}[t]
\centering
\caption{Root-node performance on CVRP test instances. Each cell reports mean optimality gap (\%, top) and mean runtime in seconds (bottom, in parentheses). Lower is better for both metrics.}
\label{tab:cvrp_3panel_iters_with_exact}
\scriptsize
\setlength{\tabcolsep}{3.0pt}
\renewcommand{\arraystretch}{1.08}

\begin{adjustbox}{max width=\textwidth}
\begin{tabular}{l|ccccccccc}
\toprule
\textbf{Method} & \textbf{50} & \textbf{75} & \textbf{100} & \textbf{200} & \textbf{300} & \textbf{400} & \textbf{500} & \textbf{750} & \textbf{1000} \\
\midrule
Exact
& \makecell{1.98\\(55.45)}
& \makecell{2.76\\(203.97)}
& \makecell{4.52\\(671.00)}
& \makecell{6.35\\(3593.72)}
& \makecell{11.98\\(3651.60)}
& \makecell{18.89\\(3698.08)}
& \makecell{24.87\\(3718.09)}
& \makecell{41.00\\(4252.43)}
& \makecell{50.61\\(4624.86)}
\\
CVRPSEP
& \makecell{1.97\\(0.22)}
& \makecell{2.77\\(1.89)}
& \makecell{4.54\\(7.12)}
& \makecell{6.28\\(212.37)}
& \makecell{7.78\\(1800.42)}
& \makecell{11.89\\(3604.61)}
& \makecell{15.68\\(3608.27)}
& \makecell{24.82\\(3602.54)}
& \makecell{28.72\\(3609.18)}
\\
NeuralSEP$_2$
& \makecell{5.25\\(4.11)}
& \makecell{5.16\\(13.64)}
& \makecell{6.41\\(35.86)}
& \makecell{8.31\\(181.47)}
& \makecell{10.09\\(893.64)}
& \makecell{13.52\\(1948.70)}
& \makecell{14.48\\(3276.36)}
& \makecell{21.83\\(3672.58)}
& \makecell{26.09\\(3740.21)}
\\
$\pi$-NeuralSEP$_2$+GC
& \makecell{3.68\\(6.12)}
& \makecell{4.40\\(16.33)}
& \makecell{5.97\\(37.28)}
& \makecell{7.69\\(282.44)}
& \makecell{8.67\\(1833.29)}
& \makecell{10.81\\(3075.62)}
& \makecell{13.11\\(3639.13)}
& \makecell{20.65\\(3723.46)}
& \makecell{25.82\\(3842.77)}
\\
SAG-Sep
& \makecell{4.37\\(0.32)}
& \makecell{5.91\\(1.22)}
& \makecell{7.11\\(2.82)}
& \makecell{9.10\\(23.01)}
& \makecell{9.42\\(343.30)}
& \makecell{\textbf{9.98}\\\textbf{(1897.26)}}
& \makecell{\textbf{11.35}\\\textbf{(3045.45)}}
& \makecell{\textbf{17.41}\\(3650.99)}
& \makecell{\textbf{22.03}\\(3655.57)}
\\
\bottomrule
\end{tabular}
\end{adjustbox}
\end{table*}

\subsection{Augmented LP Graphs for RCC Separation}
\label{sec:augmented_graph}

We isolate the effect of graph construction by fixing the separator architecture and varying the LP-graph augmentation. Table~\ref{tab:ablation-graph-9col} reports the main ablation.

\begin{table}[t]
\centering
\caption{Ablation on augmented LP graph variants: mean optimality gap (\%), lower is better. Separately trained models from an earlier development stage (augmented variants with hop radius 3); rows compare with each other, not with Table~\ref{tab:cvrp_3panel_iters_with_exact}.}
\label{tab:ablation-graph-9col}
\scriptsize
\setlength{\tabcolsep}{3.0pt}
\renewcommand{\arraystretch}{1.08}
\begin{adjustbox}{max width=\textwidth}
\begin{tabular}{l|ccccccccc}
\toprule
\textbf{Variant} & \textbf{50} & \textbf{75} & \textbf{100} & \textbf{200} & \textbf{300} & \textbf{400} & \textbf{500} & \textbf{750} & \textbf{1000} \\
\midrule
Raw
& 7.79 & 9.50 & 9.78 & 11.95 & 12.31 & 12.81 & 14.51 & 22.00 & 27.80 \\
+Conn
& 7.88 & 9.97 & 11.22 & 14.26 & 16.19 & 19.09 & 21.33 & 24.71 & 27.64 \\
+Hop
& 6.57 & 7.63 & 8.28 & 9.93 & 10.70 & 11.40 & 12.87 & 19.65 & 24.63 \\
+Val
& \textbf{6.37} & \textbf{6.41} & \textbf{7.38} & \textbf{9.00} & \textbf{10.04} & \textbf{10.88} & \textbf{12.16} & \textbf{19.23} & \textbf{24.04} \\
\bottomrule
\end{tabular}
\end{adjustbox}
\end{table}

Simply increasing connectivity is not enough: adding zero-feature synthetic edges (\texttt{+Conn}) degrades performance relative to the raw LP-support graph at eight of nine sizes, since uninformative long-range edges introduce spurious connectivity that the model cannot disentangle from genuine LP support; \emph{how} we add connectivity matters more than \emph{how much}.

Performance improves only when the augmented edges carry interpretable structure: hop-type information (\texttt{+Hop}) distinguishes native LP edges from synthetic multi-hop connections and consistently improves over the raw graph, and flow-based edge values (\texttt{+Val}) give the best results at all sizes, reflecting that RCC violations are tied to bottleneck and flow structure. We use hop radius $2$ (Appendix~\ref{sec:hop_ablation}) and a six-layer transformer (Appendix~\ref{sec:training_detail}) by default.

\subsection{Test-Time Refinement}
\label{sec:test_time_refinement_results}

We next compare SAG-Sep with training-free test-time refinements that reuse the same trained encoder. SAG-Sep+Aggr corresponds to the dense conditioned probing strategy from Appendix~\ref{app:aggressive-sampling}, combined with cut management. SAG-Sep+OS uses the Onion-Sweep decoder of Section~\ref{sec:aggressive-sampling}. SAG-Sep+Onion-MBP, our earlier margin-band decoder, is kept as an ablation (configuration in the caption of Table~\ref{tab:test_time_refinement} and Appendix~\ref{sec:hyperparameter_selection}). Paired per-instance statistics for these comparisons are reported in Appendix~\ref{sec:paired_appendix}.

\begin{center}
\captionsetup{type=table,hypcap=false}
\caption{Effect of test-time refinements: mean optimality gap (\%, top) and mean runtime in seconds (bottom); lower is better. SAG-Sep+Onion-MBP is our earlier decoder, kept as an ablation.}
\label{tab:test_time_refinement}
\scriptsize
\setlength{\tabcolsep}{3.0pt}
\renewcommand{\arraystretch}{1.08}
\begin{adjustbox}{max width=\textwidth}
\begin{tabular}{l|ccccccccc}
\toprule
\textbf{Method} & \textbf{50} & \textbf{75} & \textbf{100} & \textbf{200} & \textbf{300} & \textbf{400} & \textbf{500} & \textbf{750} & \textbf{1000} \\
\midrule
SAG-Sep
& \makecell{4.37\\(0.32)}
& \makecell{5.91\\(1.22)}
& \makecell{7.11\\(2.82)}
& \makecell{9.10\\(23.01)}
& \makecell{9.42\\(343.30)}
& \makecell{9.98\\(1897.26)}
& \makecell{11.35\\(3045.45)}
& \makecell{17.41\\(3650.99)}
& \makecell{22.03\\(3655.57)}
\\
SAG-Sep+Aggr
& \makecell{3.27\\(0.82)}
& \makecell{4.54\\(3.73)}
& \makecell{5.88\\(9.14)}
& \makecell{8.07\\(52.85)}
& \makecell{8.91\\(411.19)}
& \makecell{9.60\\(1863.35)}
& \makecell{10.77\\(3045.63)}
& \makecell{\textbf{16.99}\\(3619.31)}
& \makecell{21.70\\(3635.41)}
\\
SAG-Sep+Onion-MBP
& \makecell{3.47\\\textbf{(0.24)}}
& \makecell{4.18\\\textbf{(0.88)}}
& \makecell{5.90\\(2.46)}
& \makecell{8.01\\(27.50)}
& \makecell{8.85\\\textbf{(299.47)}}
& \makecell{9.41\\\textbf{(1523.99)}}
& \makecell{\textbf{10.66}\\\textbf{(2897.77)}}
& \makecell{17.24\\(3665.06)}
& \makecell{21.26\\(3705.87)}
\\
SAG-Sep+OS
& \makecell{\textbf{2.60}\\(0.29)}
& \makecell{\textbf{3.80}\\(0.95)}
& \makecell{\textbf{5.43}\\\textbf{(2.37)}}
& \makecell{\textbf{7.62}\\\textbf{(22.67)}}
& \makecell{\textbf{8.56}\\(400.33)}
& \makecell{\textbf{9.19}\\(2031.73)}
& \makecell{10.78\\(3105.12)}
& \makecell{17.02\\(3633.69)}
& \makecell{\textbf{21.25}\\(3673.29)}
\\
\bottomrule
\end{tabular}
\end{adjustbox}
\end{center}

\textbf{Aggressive sampling is useful but time-consuming.} It improves the bounds at all sizes, confirming that denser conditioning grids expose additional violated cuts, but encoding the additional replicas is costly on small instances; on large ones the LP loop dominates and this cost becomes marginal.

\textbf{Onion-Sweep gives the best bound at most sizes and needs no extra forward pass.} Replacing the threshold readout by OS improves the mean root bound at every size and on 85 of the 90 benchmark instances (Appendix~\ref{sec:paired_appendix}): by $1.5$--$2.1$ pp at $N\le 200$ at equal or lower runtime, and by $0.4$--$0.9$ pp at $N\ge 300$, where at the time-limit-bound sizes $N\ge 500$ repeated runs of one configuration differ by up to $0.4$ pp (Appendix~\ref{sec:hyperparameter_selection}). Compared with Onion-MBP, OS removes the threshold, band-width, and cut-management dependence while improving the bound by $0.2$--$0.9$ pp at $N\le 400$ (71 of 90 paired wins overall); at $N\ge 500$ the two are within $0.25$ pp of each other in either direction. Because OS proposes more cuts per round without cut management, its root runtime at $N=300$--$500$ exceeds the threshold readout and the cut-managed Onion-MBP. This also helps explain why CVRPSEP remains competitive on small instances: rapidly producing many diverse cuts is itself a powerful tightening mechanism (Appendix~\ref{sec:two_exact_sep}).

\section{Limitations}
\label{sec:limitations}

\textbf{What is established, and what is not.} This paper focuses on the root node of the two-index formulation, where we set out to show how much more efficient learned separation can be: SAG-Sep with OS attains lower root gaps than learned and classical separators within the same time limit in large-scale, LP-bound regimes, where LP re-optimization dominates the root loop. At $N\le 300$ CVRPSEP remains stronger, and the effect on complete branch-and-cut or branch-price-and-cut solves is only preliminarily tested (Appendix~\ref{sec:code_release}). Beyond the root the LP states change character: in exact-separation branch-and-cut runs on our $N=50$--$100$ instances, 86\% of the cuts found in the first tenth of the root rounds are capacity-dominated, none in the second half of the root loop or at tree nodes, and the mean violation falls from 7.9 to 0.13 (Appendix~\ref{sec:rcc_decomposition}); in a route-based master with capacity-feasible closed depot routes, $\bar x(\delta(S))\ge 2d(S)/C$ holds, so only the rounding part can ever be violated. Our labels and decoder target the batch-oriented root regime; the small, rounding-driven violations beyond it need an additional search layer, which we release but do not claim here.

\textbf{What transfers.} The separator consumes only the LP support graph, the demands, and the capacity, so the encoder and decoder plug, without architectural change, into any relaxation that exposes edge values, the host formulation supplying the valid inequality and its violation check: the nodes of a branch-and-cut tree, the root of a branch-price-and-cut solver (where the released integration fine-tunes the checkpoint on column-generation states), and the VRP variants that keep a single capacity (e.g., time windows and duration limits). We regard this component as the reusable part of the work, and the released branch-and-cut examples exercise it zero-shot against CVRPSEP (Appendix~\ref{sec:variant_bc}): on 30 capacitated, asymmetric and time-windowed instances the engine reaches exactly CVRPSEP's root bounds and both solve all 30; on the open VRP the bounds agree within 0.004 points of gap; on the capacitated minimum spanning tree problem, with the capacity cut as the only family and no CMST training, its root bound is higher on 13 of 20 instances (by up to 5.8 points at tight capacity) and lower on 6 (by at most 0.44). Linehaul--backhaul and pickup-and-delivery loads need variant-specific inequalities, and heterogeneous fleets remain open, since without a single $C$ our vehicle-count conditioning does not apply directly.

\textbf{Labels, benchmark, confounders.} Training relies on exact-separation labels, which are expensive at the training size; smaller-size labels recover much of the performance (Table~\ref{tab:training_size_ablation}), but label dependence limits fast transfer to other cut families, and we release the checkpoint and the label set. The main evaluation uses the NeuralSEP benchmark distribution, with held-out XML sizes and CVRPLIB X instances as partial mitigation (Appendix~\ref{sec:generalization_appendix}). Implementation quality matters as well: the main comparisons run our Python research implementation against the baselines' own code, so a same-language comparison with CVRPSEP is the fairest measure of speed; Appendix~\ref{sec:engine_vs_cvrpsep} runs the released C engine against the native C CVRPSEP in the same root loop, and Appendix~\ref{sec:n1000_exact_compare} gives a fixed-round comparison that removes implementation speed altogether. OS removes the threshold and band dependence of Onion-MBP but remains a hand-designed decoder; a learned decoder is future work.

\section{Conclusion}
We introduced SAG-Sep, a learning-based RCC separator built on a Sparse Augmented Graph representation of LP solutions that adds flow-informed long-range structure to the LP-support graph and enables single-pass neural separation. Exact separation revealed that high-value RCCs are often nested, onion-like, and the encoder reflects this in its scores; Onion-Sweep exploits it as a training-free decoder that tightens root bounds without extra forward passes or threshold tuning.

Beyond RCC separation, our results suggest that cut generation can be reframed as structured prediction over solver states, with attention architectures taking over expensive separation subroutines at the root node. Two directions follow: a learned decoder that reasons over candidate diversity and relaxation improvement could replace the hand-designed test-time logic, and reinforcement learning or online feedback could optimize entire cut sets rather than individual supervised labels.

\section*{Acknowledgments}
Haoran Liu and Yu Yang were supported by the National Science Foundation, Division of Civil,
Mechanical and Manufacturing Innovation, through grant CMMI-2309667. The authors declare no
competing interests.

\bibliographystyle{plainnat}
\bibliography{reference}

\clearpage
\appendix

\section*{Appendices}
\vspace{-0.5em}
\section*{Contents}
\vspace{-0.75em}
\begingroup
\small
\makeatletter
\newcommand{\appcontentssection}[3]{\@dottedtocline{1}{0em}{2.6em}{\textbf{#1\quad #2}}{\textbf{\pageref{#3}}}}
\newcommand{\appcontentssubsection}[3]{\@dottedtocline{2}{1.5em}{3.4em}{#1\quad #2}{\pageref{#3}}}
\makeatother
\appcontentssection{A}{Additional Methodological and Theoretical Details of SAG-Sep}{app:method-details}
\appcontentssubsection{A.1}{Transformer Layer}{sec:transformer-layer}
\appcontentssubsection{A.2}{Lipschitz Continuity of the Attention Module}{sec:lipschitz-proof}
\appcontentssubsection{A.3}{Aggressive Sampling via Conditioned Probing}{app:aggressive-sampling}
\appcontentssubsection{A.4}{Proof of the Resolution Error Bound}{app:proof-resolution}
\appcontentssubsection{A.5}{Proof of Sparse Augmentation Complexity}{app:proof-er-blowup}
\appcontentssection{B}{Exact RCC Separation Study and Onion Analysis}{sec:two_exact_sep}
\appcontentssubsection{B.1}{Three Exact RCC Separation Formulations}{subsec:exact_formulations}
\appcontentssubsection{B.2}{Diversity and Repetition Properties}{subsec:exact_diversity}
\appcontentssubsection{B.3}{Empirical Comparison of Exact Formulations}{subsec:exact_empirical_compare}
\appcontentssubsection{B.4}{Onion Structure in Exact Labels and Learned Scores}{subsec:onion_structure_analysis}
\appcontentssection{C}{Implementation and Training Details}{sec:training_detail}
\appcontentssubsection{C.1}{Training Procedure}{sec:training_data}
\appcontentssubsection{C.2}{Loss Functions}{sec:loss_functions}
\appcontentssubsection{C.3}{Shared Edge Encoding and Inference Memory}{sec:shared_edge_encoding}
\appcontentssubsection{C.4}{Hyperparameter Selection}{sec:hyperparameter_selection}
\appcontentssection{D}{Additional Discussions and Results}{sec:moredisc}
\appcontentssubsection{D.1}{NeuralSEP$_1$ Reproduction and One-Shot Base Model}{sec:neuralsep1_repro}
\appcontentssubsection{D.2}{Fixed-Round Comparisons at $N{=}1000$}{sec:n1000_exact_compare}
\appcontentssubsection{D.3}{fp32 vs.\ TF32 Precision Sensitivity}{sec:precision_sensitivity}
\appcontentssubsection{D.4}{Serial SAG-Sep + CVRPSEP Hybrid}{sec:cvrpsep_hybrid_appendix}
\appcontentssubsection{D.5}{Detailed $K$-hop Ablation}{sec:hop_ablation}
\appcontentssubsection{D.6}{Paired Per-Instance Comparisons}{sec:paired_appendix}
\appcontentssubsection{D.7}{Generalization Beyond the Benchmark}{sec:generalization_appendix}
\appcontentssection{E}{Code Release and Scope Beyond the Root Node}{sec:code_release}
\appcontentssubsection{E.1}{The Released Engine Run to Closure Against CVRPSEP}{sec:engine_vs_cvrpsep}
\appcontentssubsection{E.2}{Preliminary Complete Branch-and-Cut}{sec:bc_fulltree}
\appcontentssubsection{E.3}{Capacity and Rounding Parts of Exact RCC Violations}{sec:rcc_decomposition}
\appcontentssubsection{E.4}{Transfer Across Problem Classes: The Released Examples}{sec:variant_bc}
\endgroup

\section{Additional Methodological and Theoretical Details of SAG-Sep}
\label{app:method-details}

\subsection{Transformer Layer}
\label{sec:transformer-layer}

\begin{figure*}[htbp]
  \centering
  \includegraphics[width=\linewidth]{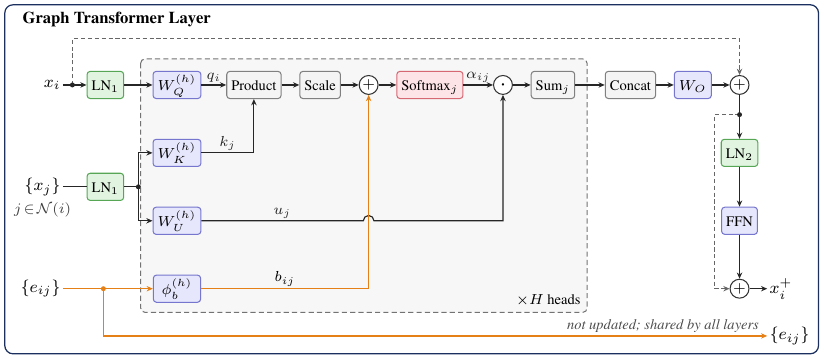}
  \caption{Graph Transformer layer with edge-biased multi-head attention (pre-normalization). Edge embeddings are encoded once, shared by all layers, and enter each head only through the scalar attention bias $b_{ij}$.}
  \label{fig:gt_layer}
\end{figure*}

The transformer layer in SAG-Sep follows the standard encoder block
\citep{vaswani2017attention} and consists of a multi-head attention (MHA) module
followed by a feed-forward network (FFN). The implementation uses residual
connections and layer normalization around these two components. This subsection
focuses on the attention module, where edge features provide an additive bias to
the scaled dot-product logits; see Figure~\ref{fig:gt_layer}.

\paragraph{Attention layer.}
Let $X\in\mathbb{R}^{n\times d}$ denote the input node embeddings of the current layer,
where $n=|V|$ and each row, denoted $x_i^\top$, corresponds to node $i$. The query, key, and value projections are
\begin{equation}
\mathbf Q = XW_Q,\quad \mathbf K = XW_K,\quad \mathbf U = XW_U,
\label{eq:qkv}
\end{equation}
where $W_Q,W_K\in\mathbb{R}^{d\times d_k}$ and $W_U\in\mathbb{R}^{d\times d_v}$ are learnable
matrices. We write $q_i^\top,k_i^\top,u_i^\top$ for the rows of $\mathbf Q,\mathbf K,\mathbf U$.

\paragraph{Edge-biased sparse attention.}
Let $\widetilde G=(V,\widetilde E)$ be the LP-induced augmented graph and
$\mathcal{N}(i)=\{j:\{i,j\}\in\widetilde E\}$ be the neighbor set of node $i$.
Attention is asymmetric (the weights $i\!\to\!j$ and $j\!\to\!i$ are computed independently), and we do not add self-loop attention edges; the residual connection at the layer level carries each node's previous representation. Every vertex has positive LP degree and the depot is joined to every customer, so no neighborhood $\mathcal{N}(i)$ is empty.
Each undirected edge $\{i,j\}\in\widetilde E$ has an encoded edge embedding $e_{ij}=e_{ji}$ obtained from the input feature $e^{\mathrm{input}}_{ij}$ in Section~\ref{sec:aug-lp-graph} by a learnable edge encoder. From this embedding we compute a scalar attention bias
\begin{equation}
b_{ij} = \phi_b(e_{ij}), \qquad \{i,j\}\in\widetilde E,
\label{eq:edge-bias}
\end{equation}
where $\phi_b(\cdot)$ is a learnable layer-specific linear map. We restrict attention to $\widetilde E$ by
normalizing over $\mathcal{N}(i)$ (equivalently, using a mask with $-\infty$ on non-edges).
In SAG-Sep, edge embeddings are used to generate layer-specific attention biases and
are not updated by the node-attention block.
The attention weights are then
\begin{equation}
\alpha_{ij}
=
\frac{
\exp\!\left(\frac{q_i^\top k_j}{\sqrt{d_k}} + b_{ij}\right)
}{
\sum\limits_{t\in\mathcal{N}(i)}
\exp\!\left(\frac{q_i^\top k_t}{\sqrt{d_k}} + b_{it}\right)
},
\qquad j\in\mathcal{N}(i),
\label{eq:edge-biased-attn}
\end{equation}
and the output is
\begin{equation}
z_i = \sum_{j\in\mathcal{N}(i)} \alpha_{ij}\, u_j,
\qquad Z = \mathrm{Attention}(X) \in \mathbb{R}^{n\times d_v}.
\label{eq:attn-out}
\end{equation}
In matrix form, letting $B\in\mathbb{R}^{n\times n}$ collect the edge biases with
$B_{ij}=\phi_b(e_{ij})$ for $j\in\mathcal{N}(i)$ and $B_{ij}=-\infty$ otherwise,
we write
\begin{equation}
\mathrm{Attention}(X)
=
\mathrm{Softmax}\!\left(\frac{\mathbf Q\mathbf K^\top}{\sqrt{d_k}} + B\right)\mathbf U,
\label{eq:attn-matrix}
\end{equation}
where $\mathrm{Softmax}(\cdot)$ is applied row-wise.

\paragraph{Multi-head attention.}
For $H$ heads, we use head-specific projections
$W_Q^{(h)},W_K^{(h)},W_U^{(h)}$ and head-specific edge-bias maps
$\phi_b^{(h)}\!:e_{ij}\mapsto b_{ij}^{(h)}$.
Each head produces $Z^{(h)}=\mathrm{Attention}^{(h)}(X)$ via
\eqref{eq:qkv}--\eqref{eq:attn-matrix}, and the MHA output is
\begin{equation}
\mathrm{MHA}(X)
=
\mathrm{Concat}\!\left(Z^{(1)},\ldots,Z^{(H)}\right)W_O,
\label{eq:mha}
\end{equation}
where $W_O\in\mathbb{R}^{(H d_v)\times d}$ is a learnable projection.

\paragraph{Residual, normalization, and FFN.}
The full layer uses a pre-normalization transformer block:
\begin{align}
    \bar X &= \mathrm{LN}_1(X), \\
    Y &= X + \mathrm{Dropout}\!\left(\mathrm{MHA}(\bar X)\right), \\
    X^+ &= Y + \mathrm{Dropout}\!\left(\mathrm{FFN}(\mathrm{LN}_2(Y))\right),
\end{align}
where $\mathrm{FFN}$ is a two-layer feed-forward network applied row-wise. The layer output is $X^+$, while the edge embeddings remain shared auxiliary features for computing attention biases.

\subsection{Lipschitz Continuity of the Attention Module}
\label{sec:lipschitz-proof}

We derive a Lipschitz bound for the edge-biased sparse attention mechanism with respect to the node features $X$. This module-level stability result supports the local stability assumption used in the conditioned-probing analysis of Appendix~\ref{app:aggressive-sampling}, which guided the sampling budget of the aggressive-sampling baseline; the Onion-Sweep decoder does not rely on it, so we only sketch the argument. We use the spectral norm $\|\cdot\|_2$ for matrices and the Euclidean norm for vectors.

\begin{assumption}[Bounded Input]
\label{ass:bounded-input}
We assume the input node features lie in a bounded set, i.e., $\|X\|_2 \le C_X$ for some constant $C_X > 0$. This local boundedness assumption is standard in self-attention analyses and is stabilized in practice by layer normalization and finite inference activations.
\end{assumption}

\begin{lemma}[Lipschitz Bound for a Single Attention Head]
\label{lem:single-head-lip}
Under Assumption~\ref{ass:bounded-input}, consider the single-head attention function $f(X): \mathbb{R}^{n\times d} \to \mathbb{R}^{n\times d_v}$ defined in Eq.~\eqref{eq:attn-matrix}. Let $L_\sigma$ be a Lipschitz constant of the row-wise Softmax map on the bounded logit set induced by $\|X\|_2\le C_X$, measured in spectral norm. Let
\begin{equation}
  C_S := \sup_X \|S(X)\|_2,
\end{equation}
where $S(X)=\mathrm{Softmax}(A(X))$ is the masked attention matrix and the supremum is taken over $\|X\|_2\le C_X$. For masked row-softmax attention on $\widetilde G$, $C_S$ is finite; for example, $C_S\le \sqrt{\Delta_{\widetilde G}}$, where
\begin{equation}
  \Delta_{\widetilde G}:=\max_{j\in V}\bigl|\{i\in V:j\in\mathcal{N}(i)\}\bigr|
\end{equation}
is the maximum column sparsity of the attention mask. When the mask is symmetric, $\Delta_{\widetilde G}$ is the maximum degree of $\widetilde G$, and in all cases $C_S\le\sqrt n$.
The Lipschitz constant of $f$ with respect to $X$, denoted as $\mathcal{L}_{\mathrm{Attn}}$, is bounded by:
\begin{equation}
    \mathcal{L}_{\mathrm{Attn}} \le \|W_U\|_2 \left( C_S + \frac{2 L_\sigma C_X^2}{\sqrt{d_k}} \|W_Q\|_2 \|W_K\|_2 \right).
\end{equation}
\end{lemma}

\begin{proof}[Proof sketch]
Write $A(X)=\frac{1}{\sqrt{d_k}}(XW_Q)(XW_K)^\top+B$ for the pre-softmax logits, $S(X)=\mathrm{Softmax}(A(X))$, and $Z=f(X)=S(X)XW_U$. For a perturbation $X\to X+\Delta$, the product rule gives $\delta Z=(\delta S)XW_U+S(\delta X)W_U$. The second term is bounded by $C_S\|\delta X\|_2\|W_U\|_2$. For the first term, the bias $B$ depends only on edge features, so $\delta B=0$ and $\|\delta A\|_2\le \frac{2\|X\|_2}{\sqrt{d_k}}\|W_Q\|_2\|W_K\|_2\|\delta X\|_2$; the row-wise Softmax Jacobian bound then yields $\|\delta S\|_2\le L_\sigma\|\delta A\|_2$. Combining the two terms bounds the operator norm of the Jacobian uniformly over $\{X:\|X\|_2\le C_X\}$, and integrating this bound along the line segment between any two points of this set gives the stated Lipschitz constant.
\end{proof}

\paragraph{Extension to multi-head attention.}
For the MHA module defined in Eq.~\eqref{eq:mha}, the output is a concatenation projected by $W_O$.
Let $\mathcal{L}^{(h)}$ be the Lipschitz constant of the $h$-th head derived in Lemma~\ref{lem:single-head-lip}.
Using $\|[A_1,\ldots,A_H]\|_2\le \sqrt{\sum_{h=1}^H\|A_h\|_2^2}$ for block concatenation, the total Lipschitz constant $\mathcal{L}_{\mathrm{MHA}}$ is:
\begin{equation}
    \mathcal{L}_{\mathrm{MHA}} \le \|W_O\|_2 \cdot \sqrt{\sum_{h=1}^H \left(\mathcal{L}^{(h)}\right)^2}.
\end{equation}
This derivation highlights that the stability of the attention module depends quadratically on the magnitude of node features $C_X$, multiplicatively on the spectral norms of the weight matrices, and on the sparse attention pattern through $C_S$.

\subsection{Aggressive Sampling via Conditioned Probing}
\label{app:aggressive-sampling}

Aggressive sampling increases the number of conditioned replicas evaluated at test time. Instead of restricting SAG-Sep to the base grid $\{m/\max(M-1,1)\}_{m=0}^{M-1}$, it samples a denser set $\mathcal W=\{w_j\}_{j=1}^{J^\star}\subset[0,1]$ and constructs one conditioned graph for each $w_j$. All replicas are evaluated in a batch, and each probability vector is thresholded to obtain
\begin{equation}
\hat S(w_j)=\{\,i\in V^+: p_i(w_j)\ge \tau\,\}.
\end{equation}
The induced RCCs are then evaluated under the current LP solution, deduplicated, and ranked by violation. This procedure can discover useful cuts missed by the base grid, but it increases both inference workload and the cost of subsequent LP re-optimization. The concrete size-adaptive grid rule used in our ablation is given in Appendix~\ref{sec:hyperparameter_selection}.

Intuitively, the base grid observes only a small number of hard-thresholded subsets from a continuous conditioning space. Between two adjacent conditioning values, many replicas remain identical if all node probabilities stay safely away from the threshold; new subsets appear only when one or more boundary nodes cross it. Aggressive sampling is therefore most useful when these unstable regions occupy a non-negligible fraction of the conditioning interval, and it naturally saturates once the grid is finer than the model's effective margin scale. This perspective also motivates the Onion-Sweep decoder in Section~\ref{sec:aggressive-sampling}, which replaces the cutoff by a violation-guided selection along the score ranking instead of evaluating a much denser conditioning grid.

The benefit of dense conditioning can be understood through local stability. Appendix~\ref{sec:lipschitz-proof} provides a module-level Lipschitz bound for the attention block; here we use the corresponding local stability assumption for the scalar conditioning channel. Let $R$ be the grid resolution, with grid points $w_r=r/R$ and intervals $I_r=[w_r,w_{r+1})$ for $r=0,\ldots,R-1$. Assume the predictor is Lipschitz continuous in the conditioning value: for all nodes $i$ and all $w,w'\in[0,1)$,
\begin{equation}
|p_i(w)-p_i(w')|\le L|w-w'|.
\end{equation}
Define the safety margin at $w_r$ as
\begin{equation}
\gamma(w_r)=\min_{i\in V^+}|p_i(w_r)-\tau|.
\end{equation}
If $\gamma(w_r)>L/R$, the thresholded subset is stable throughout $I_r$; otherwise, the interval may contain a threshold crossing. Let
\begin{equation}
\mathcal U=\bigcup_{r:\gamma(w_r)\le L/R} I_r
\end{equation}
denote the unstable region.

\begin{proposition}[Resolution Error Bound under Marginal Anti-concentration]
\label{prop:resolution-error}
Assume there exists a function $q(\cdot)$ such that, for every node $i\in V^+$, every conditioning value $w\in[0,1)$, and every $\varepsilon>0$,
\begin{equation}
\Pr\!\left(|p_i(w)-\tau|\le \varepsilon\right)\le q(\varepsilon),
\end{equation}
where the probability is over the distribution of LP states/instances and any inference randomness. The expected Lebesgue measure of the unstable region satisfies
\begin{equation}
\mathbb{E}[\mathrm{vol}(\mathcal U)]\le |V^+|\,q(L/R).
\end{equation}
\end{proposition}

This bound controls the measure of the conditioning interval in which thresholded subsets can change; it does not bound missed cuts, but it explains why dense conditioning saturates rather than improving indefinitely. Increasing the grid resolution reduces the perturbation scale \(L/R\), so once \(R\) exceeds the model's effective margin scale, the only intervals that can still produce new thresholded subsets are those containing nodes extremely close to the decision boundary. When the model has little probability mass near the threshold, dense probing quickly reaches diminishing returns rather than producing an unbounded number of useful new cuts. This is also why Onion-Sweep decodes the violation profile of the score ranking directly instead of relying on ever-denser conditioning grids.

\paragraph{Empirical check of marginal anti-concentration.}
The proposition is informative only if $q(\varepsilon)\to 0$ as $\varepsilon\to 0$. We check this empirically on the trained SAG-Sep predictor, pooling predictions over nodes and conditioning levels (the proposition assumes the bound for each of them). Using $442$ random LP states drawn from $N=300$ root relaxations, we collect roughly $6\times 10^6$ post-sigmoid predictions $p_i(w)$ across nodes, conditioning levels, and instances; Figure~\ref{fig:margin-anticoncentration} reports the resulting marginal distribution and CDF.
\begin{figure*}[htbp]
  \centering
  \includegraphics[width=0.95\linewidth]{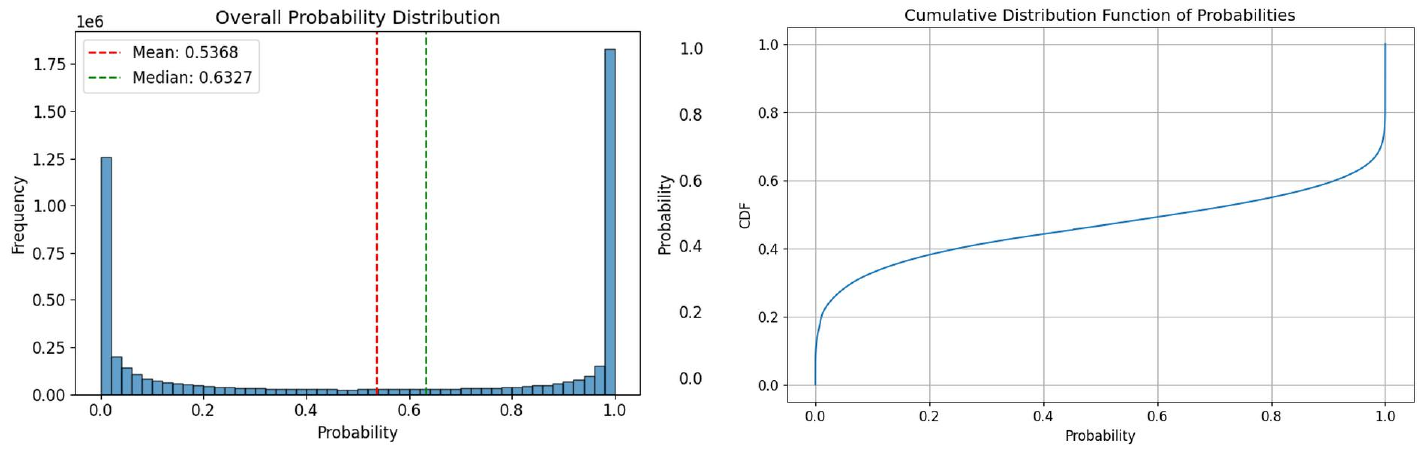}
  \caption{Empirical marginal of $p_i(w)$ on $N=300$ root-relaxation LP states (approximately $6\times 10^6$ post-sigmoid predictions). Left: histogram is sharply bimodal at $0$ and $1$ with low mass near the threshold $\tau=0.5$. Right: the CDF is shallow on the central region, giving $\Pr(|p_i(w)-\tau|\le\varepsilon)\lesssim 0.6\,\varepsilon$ for $\varepsilon\le 0.1$ and supporting $q(\varepsilon)=O(\varepsilon)$.}
  \label{fig:margin-anticoncentration}
\end{figure*}
Two observations follow. First, the predictor concentrates its mass near $0$ and $1$ with sharply lower density in the central region around $\tau$, and this bimodal pattern is consistent across instances rather than driven by a few outliers (Figure~\ref{fig:margin-anticoncentration-pergraph}). Reading the empirical CDF, $\Pr(|p_i(w)-\tau|\le\varepsilon)\lesssim 0.6\,\varepsilon$ for $\varepsilon\le 0.1$, so $q(\varepsilon)$ scales linearly with a small constant in the relevant regime. Second, the same bimodal structure explains why the hard-threshold readout is already a reasonable decoder and why its remaining failures concentrate on the few nodes near $\tau$; Onion-Sweep removes this dependence by selecting prefixes of the score ranking by violation rather than by a cutoff.
\begin{figure*}[htbp]
  \centering
  \includegraphics[width=0.95\linewidth]{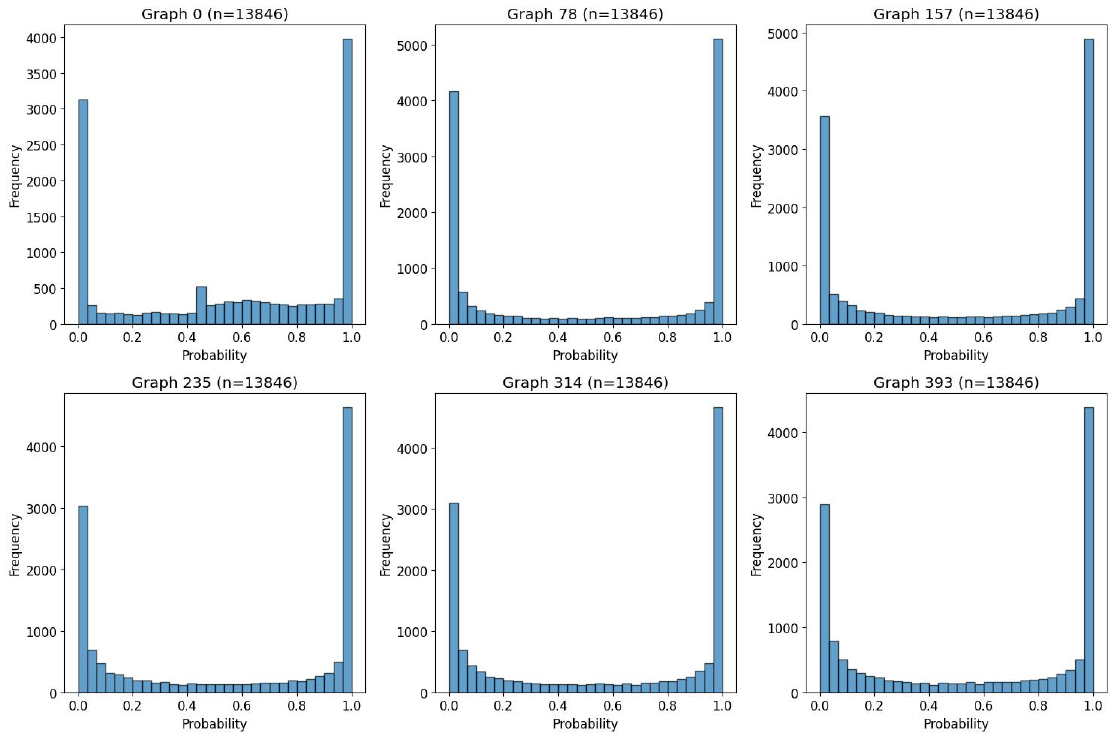}
  \caption{Per-instance marginals of $p_i(w)$ on six randomly sampled $N=300$ root-relaxation LP states. The bimodal pattern, and the resulting low density near $\tau=0.5$, is structural rather than instance-specific; isolated secondary modes (e.g.\ Graph $0$ near $p=0.45$) do not violate the linear anti-concentration rate read from the aggregate CDF.}
  \label{fig:margin-anticoncentration-pergraph}
\end{figure*}

\subsection{Proof of the Resolution Error Bound}
\label{app:proof-resolution}

\begin{proof}[Proof of Proposition~\ref{prop:resolution-error}]
Let $\delta=1/R$ and let $U_r$ be the indicator that interval $I_r$ is unstable:
\begin{equation}
    U_r = \mathbb{I}\left( \gamma(w_r) \le L\delta \right) =
    \begin{cases}
    1 & \text{if } \min_{i \in V^+} |p_i(w_r) - \tau| \le L\delta, \\
    0 & \text{otherwise}.
    \end{cases}
\end{equation}
Since the intervals are disjoint and each has measure $\delta$,
\begin{equation}
    \mathbb{E}[\mathrm{vol}(\mathcal{U})]
    = \delta \sum_{r=0}^{R-1} \Pr(U_r=1).
    \label{eq:sum_prob}
\end{equation}
For each grid point $w_r$, the unstable event is a union over nodes. By the union bound,
\begin{align}
    \Pr(U_r=1)
    &=
    \Pr\left(\bigcup_{i\in V^+}\{|p_i(w_r)-\tau|\le L\delta\}\right) \\
    &\le
    \sum_{i\in V^+}\Pr\left(|p_i(w_r)-\tau|\le L\delta\right)
    \le |V^+|\,q(L/R),
\end{align}
where the last inequality uses the marginal anti-concentration assumption. Substituting into \eqref{eq:sum_prob} gives
\begin{equation}
    \mathbb{E}[\mathrm{vol}(\mathcal{U})]
    \le \delta \sum_{r=0}^{R-1}|V^+|\,q(L/R)
     = |V^+|\,q(L/R).
\end{equation}
\end{proof}

\subsection{Proof of Sparse Augmentation Complexity}
\label{app:proof-er-blowup}

\begin{proof}[Proof of Proposition~\ref{prop:ER_blowup}]
Let $\lambda_N=\rho(N-1)$. For two distinct fixed nodes $u,v\in V$, let
$d_G(u,v)$ denote their shortest-path distance in $G_{\mathrm{LP}}$. By symmetry,
\begin{align}
    \mathbb{E}[|E_{\mathrm{LP}}|] &= \binom{N}{2}\rho, \\
    \mathbb{E}[|\widetilde E|] &= \binom{N}{2}\mathbb{P}(1\le d_G(u,v)\le K).
\end{align}
Therefore
\begin{equation}
    \mathcal{R}_N :=
    \frac{\mathbb{E}[|\widetilde E|]}{\mathbb{E}[|E_{\mathrm{LP}}|]}
    =
    \frac{\mathbb{P}(1\le d_G(u,v)\le K)}{\rho}.
    \label{eq:ratio_simplified}
\end{equation}

Let $X_t$ be the number of simple paths of length $t$ between $u$ and $v$, and
let $X_{\le K}:=\sum_{t=1}^K X_t$. Then
$\{1\le d_G(u,v)\le K\}$ is exactly the event $\{X_{\le K}\ge 1\}$.
For a fixed $t$, a simple length-$t$ path has $t-1$ ordered intermediate nodes, so
\begin{equation}
    M_t=\frac{(N-2)!}{(N-t-1)!}=(N-2)(N-3)\cdots(N-t)
\end{equation}
such paths are possible, with the product interpreted as $1$ when $t=1$. By independence of edge events under
$\mathcal{G}(N,\rho)$, each such path appears with probability $\rho^t$. Hence
\begin{equation}
    \mu_N:=\mathbb{E}[X_{\le K}]
    =\sum_{t=1}^K M_t\rho^t.
\end{equation}

Since $K$ is fixed and $\lambda_N^K=o(N)$, we have $\mu_N=o(1)$. A Bonferroni
bound gives
\begin{equation}
    \mu_N-\sum_{P<Q}\mathbb{P}(P\text{ and }Q\text{ both occur})
    \le
    \mathbb{P}(X_{\le K}\ge 1)
    \le
    \mu_N,
\end{equation}
where $P<Q$ ranges over unordered pairs of simple $u$--$v$ paths of length at most $K$. The pair
intersection term is $o(\mu_N)$ in the same sparse local regime: edge-disjoint
path pairs contribute $O(\mu_N^2)$, while distinct path pairs sharing edges form
a connected subgraph with a cycle. For each fixed overlap type, a union subgraph
with $e$ edges has at most $e-2$ free internal vertices, so its contribution is
$O(\lambda_N^e/N^2)$. Summing over the finitely many overlap types with
$e\le 2K$ gives a total shared-edge contribution of order
$O\!\left(N^{-2}(\sum_{t=1}^K \lambda_N^t)^2\right)=O(\mu_N^2)$.
Since $\mu_N=o(1)$, the full pair
intersection term is $o(\mu_N)$.
Thus
\begin{equation}
    \mathbb{P}(1\le d_G(u,v)\le K)
    =
    \mathbb{P}(X_{\le K}\ge 1)
    =
    (1+o(1))\mu_N.
\end{equation}

Substituting this estimate into \eqref{eq:ratio_simplified} yields
\begin{align}
    \mathcal{R}_N
    &=
    (1+o(1))\sum_{t=1}^K M_t\rho^{t-1} \nonumber \\
    &=
    (1+o(1))\sum_{t=1}^K
    \lambda_N^{t-1}
    \prod_{r=2}^{t}\frac{N-r}{N-1}.
\end{align}
The empty product for $t=1$ equals $1$, and for every fixed $t$ the product
converges to $1$. Therefore
\begin{equation}
    \mathcal{R}_N
    =
    (1+o(1))\sum_{t=1}^{K}\lambda_N^{t-1}
    =
    (1+o(1))\sum_{k=0}^{K-1}\lambda_N^k.
\end{equation}
If $\lambda_N\to\lambda$, this gives
$\lim_{N\to\infty}\mathcal{R}_N=\sum_{k=0}^{K-1}\lambda^k$.
\end{proof}

\section{Exact RCC Separation Study and Onion Analysis}
\label{sec:two_exact_sep}

The \textsc{Exact} baseline in Table~\ref{tab:cvrp_3panel_iters_with_exact}
uses exact separation for rounded capacity constraints (RCCs). This appendix
spells out three closely related exact formulations and explains why the
upper-bounded, demand-banded variant is more useful for generating diverse cuts
and for studying the nested structure that motivates Onion-Sweep.

\paragraph{Notation.}
Let $G=(V,E)$ be the complete graph with depot $0$ and customers
$V^+=V\setminus\{0\}$. Let $\bar{x}\in\mathbb{R}^E_{+}$ (customer edges in $[0,1]$, depot edges in $[0,2]$) denote the current fractional LP
solution, $d_i\in\mathbb{Z}_{>0}$ the demand of customer $i$, and $C$ the
vehicle capacity. For $S\subseteq V^+$, define
$D(S)=\sum_{i\in S}d_i$ and
$\bar{x}(\delta(S))=\sum_{\{i,j\}\in E:\,|\{i,j\}\cap S|=1}\bar{x}_{ij}$.
The RCC associated with $S$ is
\begin{equation}
  \bar{x}(\delta(S))\ge 2\left\lceil\frac{D(S)}{C}\right\rceil,
\end{equation}
with violation
$\nu(S)=2\left\lceil D(S)/C\right\rceil-\bar{x}(\delta(S))$.
All formulations below use binary variables $y_i$ to encode
$S=\{i\in V^+:y_i=1\}$ and nonnegative variables $\mu_{ij}$ to linearize the
cut indicator.

\subsection{Three Exact RCC Separation Formulations}
\label{subsec:exact_formulations}

\paragraph{Fukasawa lower-bound formulation.}
The classical exact separator of \citet{fukasawa2006robust} enumerates
$m\in\{0,\ldots,\lceil D(V^+)/C\rceil-1\}$ and solves

\begin{align}
  \min_{y,\mu}\;\;
    & z_F(m) = \sum_{\{i,j\}\in E} \bar{x}_{ij}\,\mu_{ij}
      \label{eq:rccsep-obj} \\
  \text{s.t.}\;\;
    & \mu_{ij} \ge y_i-y_j, && \forall \{i,j\}\in E, \label{eq:rccsep-cut1}\\
    & \mu_{ij} \ge y_j-y_i, && \forall \{i,j\}\in E, \label{eq:rccsep-cut2}\\
    & mC+1 \le \sum_{i\in V^+}d_i y_i, \label{eq:rccsep-demand-lb}\\
    & y_0=0, \label{eq:rccsep-depot}\\
    & y_i\in\{0,1\}, && \forall i\in V^+, \label{eq:rccsep-binary}\\
    & \mu_{ij}\ge 0, && \forall \{i,j\}\in E. \label{eq:rccsep-nonneg}
\end{align}
The returned subset is accepted as violated whenever $z_F(m)<2(m+1)$. This
formulation is efficient and standard, but the demand constraint only
lower-bounds the required vehicle count.

\paragraph{Direct violation-maximization variant.}
A tempting alternative is to keep the same lower-bound feasible region but
optimize the true RCC violation directly:
\begin{align}
  \max_{y,\mu}\;\;
    & z_V(m) =
      2\left\lceil
        \frac{\sum_{i\in V^+}d_i y_i}{C}
      \right\rceil
      - \sum_{\{i,j\}\in E}\bar{x}_{ij}\,\mu_{ij}
      \label{eq:rccsep-direct-violation}\\
  \text{s.t.}\;\;
    & \mu_{ij} \ge y_i-y_j, && \forall \{i,j\}\in E, \label{eq:rccsep-direct-cut1}\\
    & \mu_{ij} \ge y_j-y_i, && \forall \{i,j\}\in E, \label{eq:rccsep-direct-cut2}\\
    & mC+1 \le \sum_{i\in V^+}d_i y_i, \label{eq:rccsep-direct-demand-lb}\\
    & y_0=0, \label{eq:rccsep-direct-depot}\\
    & y_i\in\{0,1\}, && \forall i\in V^+, \label{eq:rccsep-direct-binary}\\
    & \mu_{ij}\ge 0, && \forall \{i,j\}\in E. \label{eq:rccsep-direct-nonneg}
\end{align}
The ceiling term can be linearized with an auxiliary integer vehicle-count
variable, but the compact form makes the behavior clearer: this model optimizes
the actual RCC violation while retaining the same lower-bound demand constraint.
In practice, this version has a severe diversity problem, because a high-demand
high-violation subset remains feasible for many smaller $m$ values and can
dominate the entire enumeration loop.

\paragraph{Demand-banded formulation.}
Our exact-label generation use a demand-banded variant. For
each $m$, we solve
\begin{align}
  \min_{y,\mu}\;\;
    & z_B(m) = \sum_{\{i,j\}\in E} \bar{x}_{ij}\,\mu_{ij}
      \label{eq:rccsep-banded-obj}\\
  \text{s.t.}\;\;
    & \mu_{ij} \ge y_i-y_j, && \forall \{i,j\}\in E, \label{eq:rccsep-banded-cut1}\\
    & \mu_{ij} \ge y_j-y_i, && \forall \{i,j\}\in E, \label{eq:rccsep-banded-cut2}\\
    & mC+1 \le \sum_{i\in V^+}d_i y_i \le (m+1)C,
      \label{eq:rccsep-demand-band}\\
    & y_0=0, \label{eq:rccsep-banded-depot}\\
    & y_i\in\{0,1\}, && \forall i\in V^+, \label{eq:rccsep-banded-binary}\\
    & \mu_{ij}\ge 0, && \forall \{i,j\}\in E. \label{eq:rccsep-banded-nonneg}
\end{align}
Within this band, minimizing $z_B(m)$ is equivalent to maximizing
$2(m+1)-\bar{x}(\delta(S))$. The only change from
\eqref{eq:rccsep-obj}--\eqref{eq:rccsep-nonneg} is the additional upper bound,
but it changes the structure of the returned cut set substantially.

\subsection{Diversity and Repetition Properties}
\label{subsec:exact_diversity}

\begin{lemma}[Cut-value equivalence]
\label{lem:cut_equivalence}
For any fixed binary vector $y$, an optimal choice of $\mu$ in
\eqref{eq:rccsep-cut1}--\eqref{eq:rccsep-nonneg} satisfies
$\mu_{ij}=|y_i-y_j|$ for all $\{i,j\}\in E$, and thus
$\sum_{\{i,j\}\in E}\bar{x}_{ij}\mu_{ij}=\bar{x}(\delta(S))$ for
$S=\{i\in V^+:y_i=1\}$.
\end{lemma}

\begin{proof}
Given binary $y$, constraints \eqref{eq:rccsep-cut1}--\eqref{eq:rccsep-cut2} and
$\mu_{ij}\ge 0$ imply $\mu_{ij}\ge |y_i-y_j|$. Choosing
$\mu_{ij}=|y_i-y_j|$ is feasible and, since $\bar{x}_{ij}\ge 0$, optimal (and forced on every edge with $\bar{x}_{ij}>0$).
Summing over edges gives $\bar{x}(\delta(S))$.
\end{proof}

\begin{proposition}[Repeated feasible sets under lower-bound formulations]
\label{prop:repeat_across_M}
Fix any subset $S\subseteq V^+$ with demand $D(S)$. Under any formulation that
uses only the lower-bound demand constraint \eqref{eq:rccsep-demand-lb}, the
same $S$ is feasible for every integer $m$ satisfying
\begin{equation}
  m\le \left\lfloor\frac{D(S)-1}{C}\right\rfloor .
\end{equation}
Consequently, both the Fukasawa formulation and the direct violation-maximizing
variant can return duplicated subsets across multiple $m$ values.
\end{proposition}

\begin{proof}
Feasibility with respect to $m$ is determined by $mC+1\le D(S)$, which is
equivalent to $m\le (D(S)-1)/C$. Taking the floor gives the claim.
\end{proof}

By Proposition~\ref{prop:repeat_across_M}, this repetition affects both
lower-bound formulations. It is mild but still visible in the Fukasawa
formulation: the objective minimizes only the crossing value, so a low-crossing
high-demand set can remain attractive across many lower-bound subproblems. It is
much more severe for the direct violation-maximization variant, whose objective
explicitly rewards large rounded demand; empirically this can collapse the loop
onto the same high-demand subset and makes the formulation nearly unusable as a
diverse cut generator.

\begin{proposition}[One demand band per subset]
\label{prop:unique_M}
Under the band constraint \eqref{eq:rccsep-demand-band}, any subset $S$ with
demand $D(S)$ is feasible for at most one integer $m$, namely
\begin{equation}
  m=\left\lceil\frac{D(S)}{C}\right\rceil-1 .
\end{equation}
Thus a single enumeration over $m$ cannot repeatedly admit the same subset
across multiple subproblems.
\end{proposition}

\begin{proof}
Since demands are integral, \eqref{eq:rccsep-demand-band} is equivalent to
$mC<D(S)\le (m+1)C$. Dividing by $C$ gives
$m<D(S)/C\le m+1$, which is equivalent to
$m=\lceil D(S)/C\rceil-1$.
\end{proof}

By Proposition~\ref{prop:unique_M}, the demand-banded formulation provides
diversity by construction. It does not merely change the acceptance test for
violated RCCs; it changes which subsets are available to each exact subproblem,
yielding one representative cut per demand band and preventing repeated subsets
within a separation round.

\subsection{Empirical Comparison of Exact Formulations}
\label{subsec:exact_empirical_compare}

Table~\ref{tab:exact_sep_compare} compares the classical lower-bound Fukasawa
formulation with the demand-banded variant. We omit the direct
violation-maximization variant from the table because its repeated-subset
behavior makes it ineffective as a multi-cut separator despite optimizing the
true RCC violation objective.

The banded formulation is consistently competitive and becomes stronger on the
larger instances where diversity matters most. At $N=750$ and $N=1000$, it
improves the final average gap from $41.00\%$ to $40.59\%$ and from $50.61\%$
to $50.08\%$, respectively, while also reducing runtime. This supports the
interpretation above: restricting each exact solve to a demand band avoids
wasting the cut budget on duplicated subsets and yields a more useful sequence
of exact cuts.
\begin{table*}[t]
\centering
\caption{Comparison of two usable exact RCC separation formulations: classical lower-bound Fukasawa vs.\ demand-banded diverse separation.}
\label{tab:exact_sep_compare}
\scriptsize
\setlength{\tabcolsep}{4pt}
\renewcommand{\arraystretch}{1.05}
\begin{adjustbox}{max width=\textwidth}
\begin{tabular}{l|r|rrrrrr}
\toprule
Method & Size & Avg.\ Gap ($\downarrow$) & Avg.\ LB ($\uparrow$) & Avg.\ Runtime & Avg.\ Iter. & Avg.\ Cuts & Count \\
\midrule

\multirow{9}{*}{Exact (banded, with UB)}
& 50   & \textbf{1.97\%} & 9363.38  & 54.91   & 59.20 & 227.70 & 10 \\
& 75   & 2.76\% & 13357.85 & 198.81  & 79.70 & 511.10 & 10 \\
& 100  & 4.52\% & 15940.77 & 654.56  & 125.20& 981.00 & 10 \\
& 200  & 6.36\% & 21354.73 & 3575.54 & 169.40& 2231.10& 10 \\
& 300  & \textbf{11.65\%}& 30204.47 & 3654.00 & 49.60 & 1210.40& 10 \\
& 400  & \textbf{18.69\%}& 40430.81 & 3715.66 & 23.20 & 821.80 & 10 \\
& 500  & \textbf{24.41\%}& 44997.84 & 3767.76 & 14.20 & 573.10 & 10 \\
& 750  & \textbf{40.59\%}& 54228.71 & 4059.01 & 4.20  & 278.50 & 10 \\
& 1000 & \textbf{50.08\%}& 46757.73 & 4520.46 & 2.30  & 169.00 & 10 \\
\midrule

\multirow{9}{*}{Exact (classic Fukasawa)}
& 50   & 1.98\% & 9362.89  & 55.45   & 58.00 & 222.10 & 10 \\
& 75   & 2.76\% & 13357.85 & 203.97  & 78.60 & 504.50 & 10 \\
& 100  & 4.52\% & 15940.56 & 671.00  & 122.30& 979.90 & 10 \\
& 200  & \textbf{6.35\%} & 21358.92 & 3593.72 & 169.40& 2229.90& 10 \\
& 300  & 11.98\%& 30047.31 & 3651.60 & 49.10 & 1187.20& 10 \\
& 400  & 18.89\%& 40364.69 & 3698.08 & 24.30 & 852.60 & 10 \\
& 500  & 24.87\%& 44489.87 & 3718.09 & 13.80 & 556.00 & 10 \\
& 750  & 41.00\%& 53653.12 & 4252.43 & 4.30  & 283.50 & 10 \\
& 1000 & 50.61\%& 46304.41 & 4624.86 & 2.30  & 169.00 & 10 \\
\bottomrule
\end{tabular}
\end{adjustbox}
\end{table*}

\subsection{Onion Structure in Exact Labels and Learned Scores}
\label{subsec:onion_structure_analysis}

The demand-banded formulation gives one non-duplicated subset per vehicle-count
band, making it a clean object for studying the structure of exact RCC labels.
For an LP state, let
\begin{equation}
  \mathcal S_{\mathrm{ex}}=\{S_{m_1},\ldots,S_{m_K}\},\qquad
  m_1<\cdots<m_K,
\end{equation}
be the exact subsets returned across demand bands. We report two complementary
nesting measures. \emph{Containment} averages the overlap ratio over all
unordered band pairs,
\begin{equation}
  \mathrm{Containment}(\mathcal S_{\mathrm{ex}})
  =
  \frac{2}{K(K-1)}\sum_{1\le r<s\le K}
  \frac{|S_{m_r}\cap S_{m_s}|}
       {\min\{|S_{m_r}|,|S_{m_s}|\}},
\end{equation}
and captures whether smaller exact subsets are largely contained in larger
ones across the entire band sequence. \emph{M-nesting} is the strict version
of the same quantity, the fraction of band pairs in which the smaller-band
subset is contained in the larger-band one without exception,
\begin{equation}
  \mathrm{M\text{-}nesting}(\mathcal S_{\mathrm{ex}})
  =
  \frac{2}{K(K-1)}\sum_{1\le r<s\le K}
  \mathbb{I}\{S_{m_r}\subseteq S_{m_s}\}.
\end{equation}
A value of one in either metric means perfect nesting. Containment tolerates
a few displaced nodes, whereas M-nesting counts a pair only if it is exactly
nested, so it is the stricter measure; in Table~\ref{tab:onion_exact_learned_alignment} it decreases
as $K$ grows while containment stays high.
Across the exact-label datasets, M-nesting is consistently high: on XML100
exact labels, grouped by route-length bucket, it stays between $0.836$ and
$0.868$; on the random evaluation labels it stays between $0.74$ and $0.85$ from
$N=50$ to $N=1000$ (Table~\ref{tab:onion_exact_learned_alignment}), decreasing
only mildly as the number of bands grows. The all-pair containment is even
higher, $0.88$--$0.95$ across the same sizes. This indicates that the
high-value exact RCCs are not arbitrary unrelated subsets, but ordered
``onion'' layers, and that this structure persists at the headline sizes.

We next ask whether the trained SAG-Sep encoder has learned this priority
structure. For each exact-labeled LP state, define the exact membership
frequency
\begin{equation}
  f_i=\frac{1}{K}\sum_{r=1}^K \mathbb{I}\{i\in S_{m_r}\},\qquad i\in V^+.
\end{equation}
In a perfectly nested family, high-frequency nodes are core onion nodes and
low-frequency nodes enter only in outer layers; for the approximately nested
families observed here, $f$ serves as a priority surrogate. Running the frozen SAG-Sep checkpoint on the same LP
states gives probability vectors $\{p^{(j)}\}_{j=1}^J$ over the vehicle-count
replicas. We compare the exact frequency $f$ with both the cross-replica mean
score $\bar p_i=J^{-1}\sum_j p_i^{(j)}$ and the best-aligned single replica,
using Spearman correlation. The correlations remain strong at every size:
$\rho(\bar p,f)$ is between $0.84$ and $0.89$ through $N=750$ and $0.81$ at
$N=1000$, while the best replica reaches $0.83$--$0.91$.

\begin{table}[t]
\centering
\caption{Exact onion nestedness and alignment between exact membership frequency and frozen SAG-Sep probabilities on random evaluation LP states (each MILP solved to a relative optimality gap of $10^{-4}$, on the first 20 LP states of the exact-separation root loop of each of the ten instances per size).}
\label{tab:onion_exact_learned_alignment}
\scriptsize
\setlength{\tabcolsep}{4pt}
\renewcommand{\arraystretch}{1.05}
\begin{adjustbox}{max width=\linewidth}
\begin{tabular}{l|rrrrrr}
\toprule
Size & LP states & Avg.\ $K$ & Contain. & M-nest & $\rho(\bar p,f)$ & $\rho_{\mathrm{best}}(p^{(j)},f)$ \\
\midrule
$N=50$  & 200 & 5.2  & 0.894 & 0.846 & 0.838 & 0.870 \\
$N=75$  & 200 & 8.4  & 0.881 & 0.811 & 0.864 & 0.883 \\
$N=100$ & 200 & 11.2 & 0.917 & 0.831 & 0.893 & 0.910 \\
$N=200$ & 200 & 15.8 & 0.929 & 0.806 & 0.844 & 0.862 \\
$N=300$ & 200 & 27.3 & 0.931 & 0.790 & 0.847 & 0.860 \\
$N=400$  & 200 & 39.0 & 0.947 & 0.792 & 0.853 & 0.863 \\
$N=500$  & 200 & 48.4 & 0.945 & 0.777 & 0.838 & 0.849 \\
$N=750$  & 200 & 72.7 & 0.951 & 0.765 & 0.837 & 0.850 \\
$N=1000$ & 200 & 79.1 & 0.948 & 0.737 & 0.808 & 0.828 \\
\bottomrule
\end{tabular}
\end{adjustbox}
\end{table}

These results support the design of Onion-Sweep. The exact separator produces
diverse cuts by moving through demand bands, and those cuts are largely nested.
The neural encoder already assigns higher probability to nodes that appear in
more exact onion layers, but hard thresholding each replica independently
discards much of this ranking information. Onion-Sweep recovers it at test
time: within each replica it selects the most violated prefix of the score
ranking, and across replicas it sorts nodes by the soft frequency score and
forms demand-triggered prefixes, thereby converting independently predicted
probabilities into a coordinated nested family of RCC candidates without
additional training.

\section{Implementation and Training Details}
\label{sec:training_detail}

This section documents the training pipeline used in Section~\ref{sec:exp}, including
data generation, feature construction, model configuration, and optimization.

\subsection{Training Procedure}
\label{sec:training_data}

\paragraph{XML datasets.}
The main supervised checkpoint is trained on XML100~\citep{queiroga2022xml100}.
The XML100 repository provides 10{,}000 CVRP instances whose names follow the
format \texttt{XML100\_ABCD\_EF}, where \texttt{ABCD} encodes depot
positioning, customer positioning, demand distribution, and average route
length. For each
instance, we generate exact-label supervision offline at the root relaxation:
solve the LP, apply exact RCC separation across vehicle-count levels, add the
found cuts, and repeat until no further violated RCC is found. Each intermediate
LP state yields supervised examples $(G_{\bar{x}},m,y^\star(m))$, where
$y^\star(m)$ is the exact subset returned for vehicle-count index $m$.

Exact-label generation at size 100 is expensive because the exact separator must
be solved repeatedly until the root cutting-plane loop converges, and difficult
instances can run near the one-hour cap. To study cheaper supervision, we also
generate XML20, XML30, XML40, and XML50 datasets using the same XML generation
taxonomy but smaller customer counts. These smaller instances are much cheaper
to label: XML20 is a few seconds per instance in our labeling runs, XML30 is
on the order of tens of seconds, and XML40/XML50 remain around the minute scale.
They therefore provide a practical bootstrap path when full XML100 exact
labeling is too costly.

\paragraph{Train/test split.}
All XML datasets use the same distribution-stratified split. For each XML
distribution family, the 25 instances with suffixes 1--25 are used for training,
while the remaining instances are reserved for evaluation. For model selection,
we sample a compact held-out set of 2{,}048 graphs from the instances with
suffix \texttt{26}. The split is instance-disjoint and stratified by XML family
(families are shared across splits), and the small validation set keeps
validation inexpensive. Unless otherwise noted, the main SAG-Sep
variants share the same trained checkpoint; the training-size ablation below
trains one checkpoint per XML size.

\paragraph{Features.}
For each customer node $i\in V^+$, the input feature is
$[d_i/C,\;m/\max(M-1,1)]$, where $M=\lceil d(V^+)/C\rceil$ and $m\in\{0,\ldots,M-1\}$ is
the replica index. The depot is masked out in the loss. Each edge in the
augmented LP-support graph carries the two features used in
\Cref{sec:aug-lp-graph}: the LP-flow feature $e^{\mathrm{flow}}_{ij}$ and the
LP-support distance $\mathrm{dist}_{\mathrm{LP}}(i,j)$.

\paragraph{Training settings.}
Training is run on a single NVIDIA B200 GPU paired with an Intel Xeon Platinum
8570 CPU. The XML100 checkpoint takes about one day on one B200. Smaller XML
datasets train substantially faster: XML50 is roughly a half-day run, while for
XML20 the validation loss already shows clear overfitting within about one
hour. The model and optimizer settings used for the main checkpoint are listed
in Tables~\ref{tab:train_model_config} and~\ref{tab:train_optim_config}.

The architecture is described in Appendix~\ref{sec:transformer-layer} and in
Figure~\ref{fig:gt_layer}. We summarize the training-time hyperparameters in
Table~\ref{tab:train_model_config}.

\begin{table}[t]
\centering
\caption{Model configuration for SAG-Sep training.}
\label{tab:train_model_config}
\scriptsize
\setlength{\tabcolsep}{4pt}
\renewcommand{\arraystretch}{1.05}
\begin{tabular}{l|l|l}
\toprule
Block & Item & Value \\
\midrule
\multirow{3}{*}{Backbone}
& type & transformer \\
& \#layers ($L$) & 6 \\
& latent dimension ($d$) & 256 \\
\midrule
\multirow{3}{*}{Transformer}
& \#heads & 8 \\
& dropout & 0.1 \\
& edge features & enabled \\
\midrule
\multirow{3}{*}{Policy head}
& MLP hidden sizes & 128, 64 \\
& output activation & Identity \\
& output & per-node logit \\
\bottomrule
\end{tabular}
\end{table}

We use batch size 16. The optimizer and scheduler are reported in
Table~\ref{tab:train_optim_config}.

\begin{table}[t]
\centering
\caption{Optimizer and scheduler configuration.}
\label{tab:train_optim_config}
\scriptsize
\setlength{\tabcolsep}{4pt}
\renewcommand{\arraystretch}{1.05}
\begin{tabular}{l|l|l}
\toprule
Component & Item & Value \\
\midrule
\multirow{4}{*}{Adam}
& learning rate & $1\times 10^{-3}$ \\
& weight decay & 0.0 \\
& betas & (0.9, 0.999) \\
& eps & $1\times 10^{-8}$ \\
\midrule
\multirow{3}{*}{CosineAnnealingWarmRestarts}
& $T_0$ (optimizer steps) & 200 \\
& $T_{\text{mult}}$ & 1 \\
& $\eta_{\min}$ & $1\times 10^{-5}$ \\
\bottomrule
\end{tabular}
\end{table}

\paragraph{Training data size ablation.}
Table~\ref{tab:training_size_ablation} evaluates how the size of exact-labeled
training instances affects downstream root separation. Smaller exact-labeled
datasets are weaker than XML100, as expected, but the degradation is gradual:
XML40/XML50 already recover much of the XML100 performance while requiring far
cheaper label generation. This supports using small exact-labeled instances as
a practical supervised bootstrap.

\begin{table}[t]
\centering
\caption{Root-gap ablation by exact-labeled training instance size. Values are average root gaps (\%, lower is better) on the random benchmark. Run at an earlier development stage; absolute gaps differ from Table~\ref{tab:cvrp_3panel_iters_with_exact}, and comparisons are within the table.}
\label{tab:training_size_ablation}
\scriptsize
\setlength{\tabcolsep}{4pt}
\renewcommand{\arraystretch}{1.05}
\begin{tabular}{l|ccc}
\toprule
Training data & $N=50$--$100$ & $N=200$--$400$ & $N=500$--$1000$ \\
\midrule
XML20  & 8.79 & 14.03 & 21.91 \\
XML30  & 7.25 & 12.75 & 21.26 \\
XML40  & 6.83 & 10.65 & 18.76 \\
XML50  & 6.05 & 10.06 & 18.10 \\
XML100 & \textbf{5.88} & \textbf{9.43} & \textbf{17.35} \\
\bottomrule
\end{tabular}
\end{table}

\subsection{Loss Functions}
\label{sec:loss_functions}
Let a training sample be indexed by $(g,m)$, where $g$ denotes the LP-support graph
and $m$ is the rounding index. Let $V^+(g)$ be the customer set and let
$y_i\in\{0,1\}$ be the teacher label indicating whether customer $i$ belongs to the
optimal separated subset for this $(g,m)$. The model outputs a logit $s_i\in\mathbb{R}$
per customer; denote $p_i=\sigma(s_i)$.

\paragraph{Standard binary cross-entropy (BCE).}
We use the depot-masked BCE averaged over customers:
\begin{align}
\mathcal{L}_{\mathrm{BCE}}
\;=\;
\frac{1}{|V^+(g)|}\sum_{i\in V^+(g)}
\Big(
- y_i \log p_i \;-\; (1-y_i)\log(1-p_i)
\Big).
\label{eq:bce}
\end{align}
In our final configuration, we do \emph{not} apply additional positive-class weighting;
a sweep over weighting strategies did not improve downstream separation quality
(Table~\ref{tab:posweight_sweep}).

\paragraph{Optional positive-weighted BCE (ablation).}
For completeness, we implemented a per-$m$ bucket weighting scheme motivated by
class imbalance. For a mini-batch, define the number of positive and negative
labels (with depot masked) as $n_+$ and $n_-$. The positive-weighted loss is
\begin{align}
\mathcal{L}_{\mathrm{PW}}
\;=\;
\frac{1}{|V^+(g)|}\sum_{i\in V^+(g)}
\Big(
- \alpha\, y_i \log p_i \;-\; (1-y_i)\log(1-p_i)
\Big),
\qquad
\alpha \;=\; \frac{n_-}{\max\{n_+,1\}}.
\label{eq:posweight_bce}
\end{align}
When multiple $m$ values appear in a batch, we apply \eqref{eq:posweight_bce} within
each $m$ bucket and average across buckets. We also support a mixture between
instance-equal averaging and sample-proportional averaging; empirically, these
variants did not outperform \eqref{eq:bce} in our setting.

\begin{table}[t]
\centering
\caption{Ablation on positive-weight loss: mean optimality gap (\%) by instance size (lower is better). Run at an earlier development stage; absolute gaps differ from Table~\ref{tab:cvrp_3panel_iters_with_exact}, and comparisons are within the table.}
\label{tab:posweight_sweep}
\scriptsize
\setlength{\tabcolsep}{4pt}
\renewcommand{\arraystretch}{1.05}
\begin{tabular}{lccccc}
\toprule
Setting & $N=50$ & $N=100$ & $N=200$ & $N=300$ & $N=400$ \\
\midrule
Sample loss        & 9.47 & 10.86 & 13.22 & 13.39 & 14.85 \\
Mix                & 8.55 & 10.74 & 12.38 & 12.72 & 12.98 \\
Instance loss      & 9.88 & 11.39 & 13.74 & 14.22 & 14.57 \\
No positive weight & \textbf{7.98} & \textbf{9.78} & \textbf{11.91} & \textbf{12.33} & \textbf{12.79} \\
\bottomrule
\end{tabular}
\end{table}

\subsection{Shared Edge Encoding and Inference Memory}
\label{sec:shared_edge_encoding}

At inference time, SAG-Sep evaluates $M$ vehicle-count replicas for the same LP
solution. These replicas differ only in their node-level conditioning value
$m/\max(M-1,1)$; the augmented edge set $\widetilde E$ and all edge input features are
identical. A standard batched graph implementation materializes $M$ disjoint
copies of the encoded edge tensor and the per-edge attention tensors, requiring
$O(M|\widetilde E|D)$ edge-side activation memory for hidden dimension $D$.

Our shared-edge implementation stores one edge index tensor, one encoded edge
tensor in $\mathbb{R}^{|\widetilde E|\times D}$, and a per-layer edge-bias
tensor in $\mathbb{R}^{|\widetilde E|\times H}$, where $H$ is the number of
attention heads. These edge quantities are reused across the replica dimension,
while only the node tensor is stacked over the $M$ conditioning values. This
reduces the edge-side memory to $O(|\widetilde E|D)$, up to the smaller
$O(|\widetilde E|H)$ bias term, and avoids recomputing edge encodings across
replicas.

The benefit is most visible on large instances. On an L4 GPU, the standard
full-batch implementation without shared edges runs out of memory on many
large-vehicle instances, while shared-edge inference completes all cases in the
same evaluation set (Table~\ref{tab:shared_edge_alive}). On medium instances
where both paths run, shared-edge inference is essentially cost-neutral and is
slightly faster at $N=200$ and $N=300$, while $N=400$ is dominated by
time-limit trajectory noise. On the smallest instances ($N\le 75$) the plain
batched forward pass is faster, and the SAG-Sep and SAG-Sep+OS root times at
these sizes in Tables~\ref{tab:cvrp_3panel_iters_with_exact}
and~\ref{tab:test_time_refinement} use it; the bounds are identical either way.

\begin{table}[t]
\centering
\caption{Inference feasibility with and without shared-edge encoding on an L4 GPU. Entries report non-OOM instances out of 10 random benchmark instances.}
\label{tab:shared_edge_alive}
\scriptsize
\setlength{\tabcolsep}{4pt}
\renewcommand{\arraystretch}{1.05}
\begin{tabular}{l|ccc}
\toprule
Size & Full batch, no shared edge & Microbatch 48 & Shared edge \\
\midrule
$N=500$  & 8/10 & 9/10 & \textbf{10/10} \\
$N=750$  & 4/10 & 5/10 & \textbf{10/10} \\
$N=1000$ & 3/10 & 6/10 & \textbf{10/10} \\
\bottomrule
\end{tabular}
\end{table}

Direct memory profiling confirms the source of the improvement. Even on the
minimum-vehicle instance at each size, shared-edge inference reduces peak
allocated memory by $7.4\times$--$13.1\times$ relative to the standard
full-batch path (Table~\ref{tab:shared_edge_memory}). The reduction grows with
instance size because $|\widetilde E|$ grows faster than $|V|$ under graph
augmentation.

\begin{table}[t]
\centering
\caption{Peak allocated GPU memory for full-batch inference without shared edges versus shared-edge inference.}
\label{tab:shared_edge_memory}
\scriptsize
\setlength{\tabcolsep}{4pt}
\renewcommand{\arraystretch}{1.05}
\begin{tabular}{l|l|ccc}
\toprule
Size & Profiled instance & No shared edge & Shared edge & Reduction \\
\midrule
$N=500$  & random-19-X-n500   & 2.50 GiB  & 0.34 GiB & 7.4$\times$ \\
$N=750$  & random-8-X-n750    & 6.46 GiB  & 0.66 GiB & 9.8$\times$ \\
$N=1000$ & random-19-X-n1000  & 15.68 GiB & 1.20 GiB & 13.1$\times$ \\
\bottomrule
\end{tabular}
\end{table}

\subsection{Hyperparameter Selection}
\label{sec:hyperparameter_selection}

\paragraph{Graph and subset thresholds.}
SAG-Sep uses two simple thresholds at inference time. The LP-support threshold
$\tau_1$ filters edges with LP value below $\tau_1$ before constructing the
support graph, while the policy threshold $\tau_2$ converts node probabilities
into a subset, $\hat S=\{i\in V^+:p_i\ge \tau_2\}$; this $\tau_2$ corresponds
to the threshold $\tau$ used in Section~\ref{sec:aggressive-sampling} and
Appendix~\ref{app:aggressive-sampling}. Unless stated otherwise, all
experiments use $\tau_1=10^{-6}$ and $\tau_2=0.50$.
Table~\ref{tab:threshold_sensitivity} reports a one-at-a-time sensitivity sweep.

\begin{table}[t]
\centering
\caption{Sensitivity to graph and subset thresholds. Each cell reports average root gap (\%, lower is better). Run at an earlier development stage; absolute gaps differ from Table~\ref{tab:cvrp_3panel_iters_with_exact}, and comparisons are within the table.}
\label{tab:threshold_sensitivity}
\scriptsize
\setlength{\tabcolsep}{4pt}
\renewcommand{\arraystretch}{1.05}
\begin{tabular}{l|cc|ccc}
\toprule
Setting & $\tau_1$ & $\tau_2$ & $N=50$--$100$ & $N=200$--$400$ & $N=500$--$1000$ \\
\midrule
default & $10^{-6}$ & 0.50 & \textbf{5.88} & 9.43 & 17.35 \\
$\tau_1=0.05$ & 0.05 & 0.50 & 5.91 & \textbf{9.42} & 17.58 \\
$\tau_1=0.10$ & 0.10 & 0.50 & 6.01 & \textbf{9.42} & 17.84 \\
$\tau_2=0.45$ & $10^{-6}$ & 0.45 & 6.31 & 9.45 & \textbf{17.28} \\
$\tau_2=0.55$ & $10^{-6}$ & 0.55 & 6.23 & 9.48 & 17.93 \\
\bottomrule
\end{tabular}
\end{table}

The sweep shows that threshold changes can move the final bound, but no
alternative is uniformly better. The subset threshold is especially tied to
near-boundary probabilities: shifting $\tau_2$ changes which marginal nodes are
included, and the violation-guided sweep of Onion-Sweep removes this dependence
by selecting prefixes of the score ranking by their violation
(Section~\ref{sec:aggressive-sampling}). Because these thresholds are not very interpretable and can be
over-tuned to a benchmark, we keep the fixed defaults $\tau_1=10^{-6}$ and
$\tau_2=0.50$ in all non-ablation experiments.

\paragraph{Aggressive-sampling grid.}
For aggressive sampling, we densify the conditioning grid without retraining the
model. Let $v=\lceil d(V^+)/C\rceil$ denote the vehicle lower bound used by the
current LP state. We use
\begin{equation}
J_{\mathrm{aggr}}
=
\left\lceil
\max\{10\sqrt{v},\,1.25v\}
\right\rceil
\label{eq:aggressive_grid_size}
\end{equation}
conditioning values, evenly spaced over the normalized interval. The
max rule creates a natural crossover at $v=(10/1.25)^2$: below this point, the
$\sqrt v$ term acts as a heuristic boundary-scale proxy. In geometric CVRP
instances, a compact customer region with capacity volume $v$ has perimeter on
the order of $\sqrt v$, and RCC usefulness is governed by the boundary
$\delta(S)$ rather than by volume alone. Above the crossover, the $1.25v$ term
keeps the grid tied to the natural vehicle-count scale. This replaces a fixed
large grid with a size-adaptive one: it avoids injecting too many noisy marginal
subsets on small instances, while leaving enough capacity for larger states where
additional valid cuts can be absorbed by the relaxation. Other grids are
possible, but Onion-Sweep matches or improves on aggressive sampling at every
size except $N=500$ and $N=750$ (where it is within $0.03$ pp) without any
extra forward pass; for this reason, a broader grid ablation is less
informative than the structured decoding studied in
Section~\ref{sec:aggressive-sampling}.

\paragraph{Onion-MBP procedure.}
Our earlier decoder, Onion-MBP (margin-band probing), kept the hard threshold
$\tau$ and probed a band of width $\Delta$ around it in every replica, combined
with the same cross-replica onion step as Onion-Sweep
(Section~\ref{sec:aggressive-sampling}). Its candidate generation is as follows.

\begin{center}
\fbox{\begin{minipage}{0.96\linewidth}
\small
\textbf{Onion-MBP candidate generation.}
Input: probabilities $\{p^{(j)}\}_{j=1}^J$, threshold $\tau$, band width $\Delta$, demands $d$, capacity $C$, and current LP solution $\bar x$.
\begin{enumerate}[leftmargin=1.5em,itemsep=1pt,topsep=2pt]
\item Initialize candidate pool $\mathcal C\leftarrow\emptyset$.
\item \textbf{Per-replica MBP.} For each replica $j$, define the core and margin band
\[
  C_j=\{i\in V^+:p_i^{(j)}>\tau+\Delta\},\qquad
  B_j=\{i\in V^+:|p_i^{(j)}-\tau|\le\Delta\}.
\]
Add to $\mathcal C$ the hard-threshold subset $H_j=C_j\cup\{i\in B_j:p_i^{(j)}\ge\tau\}$ and four boundary probes: $C_j\cup B_j$, $C_j$ plus a random half of $B_j$, $C_j$ plus the high-probability half of $B_j$, and $C_j$ plus the low-probability half of $B_j$.
\item \textbf{Cross-replica onion.} Compute a soft frequency score $a_i=\frac{1}{J}\sum_{j=1}^J p_i^{(j)}$, sort customers as $v_1,\ldots,v_{|V^+|}$ by decreasing $a_i$, form nested prefixes $P_k=\{v_1,\ldots,v_k\}$, and add $P_k$ to $\mathcal C$ whenever $r_k=\lceil d(P_k)/C\rceil$ reaches a new vehicle-count level.
\item \textbf{Finalize subsets.} Filter $\mathcal C$ and keep a proposed subset $S$ only if $2\lceil d(S)/C\rceil-\bar x(\delta(S))>0$.
\end{enumerate}
Output: violated RCC-defining subsets in $\mathcal C$.
\end{minipage}}
\end{center}

MBP treats the nodes far above the threshold as a stable core and explores only
the uncertain band around $\tau$, where small probability shifts can change the
induced RCC. It targets the same source of missed cuts as the aggressive
conditioned probing of Appendix~\ref{app:aggressive-sampling}, namely subsets
that appear between the base conditioning values, without evaluating a denser
conditioning grid.

\paragraph{Onion-MBP ablation.}
The ablation row in Table~\ref{tab:test_time_refinement} uses
$\Delta=0.15$ with cut management (at most 50 cuts added per round, an
active-cut limit of 1000, and removal after 15 inactive rounds) for
$N\le 500$, and the threshold readout plus the onion step, without band probes,
for $N\ge 750$, where the active-cut limit hurt the bound. In preliminary runs the band width was not a sensitive parameter
($\Delta=0.10$ and $\Delta=0.15$ were within about $0.1$ pp of each other at all
sizes up to $N=400$), and beyond $N=75$ the band probes added little over the
onion step alone, which motivated replacing them by the violation-guided sweep.
\section{Additional Discussions and Results}
\label{sec:moredisc}

This appendix collects supplementary experiments that contextualize the main
results: a NeuralSEP$_1$ reproduction and a one-shot baseline (\S\ref{sec:neuralsep1_repro}),
fixed-round comparisons at $N=1000$
(\S\ref{sec:n1000_exact_compare}), an fp32 vs.\ TF32 robustness check
(\S\ref{sec:precision_sensitivity}), a serial SAG-Sep + CVRPSEP hybrid
(\S\ref{sec:cvrpsep_hybrid_appendix}), a $K$-hop ablation
(\S\ref{sec:hop_ablation}), paired per-instance statistics
(\S\ref{sec:paired_appendix}), and transfer beyond the benchmark
(\S\ref{sec:generalization_appendix}).

\subsection{NeuralSEP$_1$ Reproduction and One-Shot Base Model}
\label{sec:neuralsep1_repro}

This subsection contextualizes SAG-Sep against the original NeuralSEP$_1$
pipeline by reproducing both its one-shot base model and its recursive
coarsening variant. The reproduction runs on a legacy machine that differs from
our main-benchmark hardware, but all rows in
Table~\ref{tab:neuralsep1_oneshot_table} use the same machine, software stack,
and solver configuration, so within-table comparisons are controlled.

\paragraph{Reproduction setup.}
We reproduce the NeuralSEP$_1$ one-shot baseline using the authors'
open-sourced code and best released one-shot checkpoint. The reproduction
environment uses the original (older) software stack (DGL with CUDA~11.3) and
runs on a legacy machine with an Intel Core i9-10980XE CPU at 3.00\,GHz and
two NVIDIA A5000 GPUs. We use a simple 2-GPU inference split and keep the
remaining code and solver settings unchanged.

\paragraph{One-shot base model.}
The one-shot row of Table~\ref{tab:neuralsep1_oneshot_table} confirms that
NeuralSEP$_1$'s base model is reasonably competitive on its training-size
regime ($N=50$--$100$) but degrades sharply with instance size: at $N=750$ and
$N=1000$ it produces essentially zero violated cuts, terminates the cutting
loop early, and leaves large optimality gaps. The robustness of the full
NeuralSEP$_1$ pipeline on large instances therefore comes from recursive
coarsening rather than from a strong one-shot predictor.

\paragraph{Coarsening reference.}
The lower row of Table~\ref{tab:neuralsep1_oneshot_table} reports the
best-performing recursive coarsening variant on the same machine. Coarsening
substantially improves robustness on large instances but incurs additional
overhead on small ones. Crucially, even the coarsening variant remains far
behind SAG-Sep at $N\ge 500$ in the main results, indicating that SAG-Sep's
advantage on large instances is not only a runtime artifact.

\begin{table*}[t]
\centering
\caption{NeuralSEP$_1$ reproduction on a legacy machine (Intel i9-10980XE + 2$\times$A5000) using the released code and checkpoints. Both rows use the same hardware and solver setting; reported metrics are averages over instances.}
\label{tab:neuralsep1_oneshot_table}
\scriptsize
\setlength{\tabcolsep}{4pt}
\renewcommand{\arraystretch}{1.05}
\begin{adjustbox}{max width=\textwidth}
\begin{tabular}{l|r|rrrrrr}
\toprule
Method & Size & Avg.\ Gap ($\downarrow$) & Avg.\ LB ($\uparrow$) & Avg.\ Runtime & Avg.\ Iter. & Avg.\ Cuts & Count \\
\midrule

\multirow{9}{*}{NeuralSEP$_1$ one-shot}
& 50   & 8.76\%  & 8719.81  & 6.65   & 23.40  & 69.10   & 10 \\
& 75   & 9.17\%  & 12612.16 & 27.80  & 48.20  & 244.50  & 10 \\
& 100  & 10.13\% & 15067.78 & 66.89  & 68.70  & 407.80  & 10 \\
& 200  & 13.39\% & 19780.76 & 433.91 & 103.90 & 748.00  & 10 \\
& 300  & 17.18\% & 28617.47 & 1928.83& 91.20  & 1534.60 & 10 \\
& 400  & 24.36\% & 37921.09 & 2778.64& 85.60  & 624.40  & 10 \\
& 500  & 40.87\% & 36222.52 & 748.99 & 21.80  & 69.10   & 10 \\
& 750  & 58.42\% & 37319.85 & 91.98  & 1.00   & 0.00    & 10 \\
& 1000 & 59.73\% & 37986.40 & 149.50 & 1.00   & 0.00    & 10 \\
\midrule

\multirow{9}{*}{NeuralSEP$_1$ coarsen (best)}
& 50   & 4.10\%  & 9150.81  & 79.57  & 39.60  & 154.60  & 10 \\
& 75   & 4.95\%  & 13087.16 & 175.38 & 60.50  & 369.80  & 10 \\
& 100  & 6.81\%  & 15578.66 & 346.04 & 88.50  & 609.90  & 10 \\
& 200  & 8.94\%  & 20780.08 & 1125.99& 134.60 & 1230.90 & 10 \\
& 300  & 10.65\% & 30957.90 & 2627.44& 122.10 & 2149.20 & 10 \\
& 400  & 13.78\% & 43304.17 & 3539.18& 81.20  & 2248.10 & 10 \\
& 500  & 16.27\% & 50434.53 & 3657.06& 53.50  & 2024.80 & 10 \\
& 750  & 23.02\% & 70728.66 & 3784.96& 26.60  & 1785.20 & 10 \\
& 1000 & 28.33\% & 67667.39 & 3856.85& 18.90  & 1431.00 & 10 \\
\bottomrule
\end{tabular}
\end{adjustbox}
\end{table*}

\subsection{Fixed-Round Comparisons at $N{=}1000$}
\label{sec:n1000_exact_compare}

To isolate separation quality from the time-limit cutoff that dominates large-$N$
runs, we compare exact separation and SAG-Sep on the same ten $N=1000$
instances under a fixed budget of ten iterations of the cutting-plane loop. Table~\ref{tab:n1000_exact_sag_ten_round}
reports the final mean gap and end-to-end runtime. SAG-Sep essentially matches
the exact separator at this resolution---only a $0.003$ percentage-point gap
difference---while reducing the mean end-to-end runtime by about $870\times$.

\begin{table}[t]
\centering
\caption{Ten-iteration comparison between exact separation and SAG-Sep on ten $N=1000$ instances (the final bound is the LP after nine rounds of cuts).}
\label{tab:n1000_exact_sag_ten_round}
\scriptsize
\setlength{\tabcolsep}{4pt}
\renewcommand{\arraystretch}{1.05}
\begin{tabular}{lccc}
\toprule
Method & Final gap (\%) & Runtime (s) & Instance wins \\
\midrule
Exact separation & 35.411 & 23051.54 & 5 / 10 \\
SAG-Sep & 35.414 & 26.48 & 5 / 10 \\
\bottomrule
\end{tabular}
\end{table}

The speed difference comes primarily from cut generation. Across the same
ten-iteration runs, SAG-Sep spends $0.65$s per separation call on average,
including $0.53$s for the encoder forward pass, whereas exact separation
spends most of its runtime inside the cut-generation routine. Over their first
ten rounds on the same instances in the runs of Table~\ref{tab:cvrp_3panel_iters_with_exact},
the NeuralSEP pipelines spend $15.98$s (NeuralSEP$_2$), $18.95$s
($\pi$-NeuralSEP$_2$) and $19.31$s ($\pi$-NeuralSEP$_2$+GC) per separation
round, $25$--$30$ times more than SAG-Sep, and CVRPSEP spends $4.49$s.

\paragraph{Equal-round comparison of decoders.}
To separate cut quality from inference speed and hardware, Table~\ref{tab:n1000_equal_round} reports the mean gap after each separation round on the same ten $N=1000$ instances, starting from the same raw LP ($59.73\%$ at round~0) and ignoring time. Measured per round, the threshold readout stays within $1$ pp of the exact per-band separator, OS overtakes it from round~5 on, and $\pi$-NeuralSEP$_2$+GC is the weakest at every round. OS adds many more cuts per round (about $150$ at this size, versus one per demand band for the exact separator), which suggests that its per-round advantage comes mainly from larger batches; we do not measure per-cut quality separately.

\begin{table}[t]
\centering
\caption{Mean root gap (\%) after each separation round on the ten $N=1000$ benchmark instances, from the same raw LP and independent of time. The exact column is the exact separator of Table~\ref{tab:n1000_exact_sag_ten_round}; the threshold column is the SAG-Sep run of that table, whose final bound is round~9; $\pi$-NeuralSEP$_2$+GC uses the re-run batch of Appendix~\ref{sec:paired_appendix}.}
\label{tab:n1000_equal_round}
\scriptsize
\setlength{\tabcolsep}{5pt}
\renewcommand{\arraystretch}{1.05}
\begin{tabular}{r|rrrr}
\toprule
Round & Threshold & OS & $\pi$-NeuralSEP$_2$+GC & Exact \\
\midrule
1 & 50.34 & \textbf{49.99} & 51.05 & 50.19 \\
2 & 47.72 & \textbf{47.58} & 48.45 & 47.72 \\
3 & 45.60 & 45.55 & 46.01 & \textbf{45.33} \\
4 & 43.49 & 42.93 & 43.81 & \textbf{42.78} \\
5 & 41.66 & \textbf{40.56} & 42.03 & 40.87 \\
6 & 39.93 & \textbf{38.57} & 40.27 & 39.23 \\
7 & 38.18 & \textbf{36.96} & 38.89 & 37.85 \\
8 & 36.77 & \textbf{35.56} & 37.35 & 36.52 \\
9 & 35.41 & \textbf{34.33} & 35.81 & 35.41 \\
\bottomrule
\end{tabular}
\end{table}

\subsection{fp32 vs.\ TF32 Precision Sensitivity}
\label{sec:precision_sensitivity}

We compare full fp32 inference against TF32-enabled inference on the same
benchmark, checkpoint, and LP setting. Table~\ref{tab:tf32_sensitivity}
summarizes four matched configurations: single-B200 fp32/TF32 and
single-L4 fp32/TF32. TF32 does not yield a uniform bound--time improvement.
On B200 it is sometimes slightly faster on small instances, but the final gap
can also worsen; on L4 it worsens the gap at $N=50$ and $N=75$ and is within
$0.35$ pp elsewhere, slightly better at $N=200$ and $N\ge 500$. The effect
is not a pure throughput issue: thresholded subset prediction changes when a
few node scores move across the $\tau=0.5$ cutoff---the same boundary
instability that Onion-Sweep removes by decoding prefixes of the score ranking
by violation (Section~\ref{sec:aggressive-sampling}). For this reason, all main results use
full fp32 inference.

\begin{table*}[t]
\centering
\caption{Full fp32 versus TF32 on the same root benchmark. Each cell reports mean root gap (\%) with mean runtime in seconds in parentheses. Shared-edge encoding is on at all sizes, so the L4 fp32 times at $N\le 75$ differ from Table~\ref{tab:cvrp_3panel_iters_with_exact} (Appendix~\ref{sec:shared_edge_encoding}).}
\label{tab:tf32_sensitivity}
\scriptsize
\setlength{\tabcolsep}{4pt}
\renewcommand{\arraystretch}{1.05}
\begin{adjustbox}{max width=\textwidth}
\begin{tabular}{lccccccccc}
\toprule
Setting & $N=50$ & $N=75$ & $N=100$ & $N=200$ & $N=300$ & $N=400$ & $N=500$ & $N=750$ & $N=1000$ \\
\midrule
B200, fp32 & 4.37(0.60) & 5.91(1.80) & 7.11(3.50) & 9.10(29.50) & 9.42(411.30) & 9.98(2066.90) & 11.20(3046.10) & 17.37(3627.00) & 22.19(3680.30) \\
B200, TF32 & 5.35(0.50) & 6.00(1.60) & 7.16(3.10) & 8.94(29.90) & 9.46(411.90) & 9.95(2138.90) & 11.62(3320.00) & 17.36(3636.50) & 22.18(3659.80) \\
L4, fp32 & 4.37(0.50) & 5.91(1.40) & 7.11(2.80) & 9.10(23.00) & 9.42(343.30) & 9.98(1897.30) & 11.35(3045.50) & 17.41(3651.00) & 22.03(3655.60) \\
L4, TF32 & 5.29(0.40) & 6.13(1.40) & 7.10(2.80) & 8.96(25.30) & 9.43(449.70) & 9.99(2108.00) & 11.17(2934.00) & 17.25(3637.40) & 21.68(3635.20) \\
\bottomrule
\end{tabular}
\end{adjustbox}
\end{table*}

\subsection{Serial SAG-Sep + CVRPSEP Hybrid}
\label{sec:cvrpsep_hybrid_appendix}

We evaluate a simple serial hybrid in which SAG-Sep runs first and CVRPSEP is
activated only after the learned separator stalls. This is a supplementary
analysis rather than a main comparison, because it augments the learned
separator with a mature classical routine instead of testing SAG-Sep on its
own. Table~\ref{tab:hybrid_cvrpsep_appendix} compares this hybrid against the
baseline SAG-Sep root loop on $N=50$--$400$. The hybrid consistently tightens
the root bound, especially on small and medium instances, but the extra LP and
heuristic work becomes expensive at larger sizes. Against the pure CVRPSEP
results in Table~\ref{tab:cvrp_3panel_iters_with_exact}, the hybrid is
especially attractive from $N=200$ onward: at $N=200$ it is within $0.01$ pp
of CVRPSEP's gap in about a third of the runtime, and at $N=300$ and $N=400$ it
reaches lower gaps in less time. The tradeoff is therefore favorable when
near-exact root bounds matter. The same setup also suggests a unified
bound--time evaluation protocol---run each separator until a fixed root-bound
target and compare runtimes---which we leave to future work.

\begin{table}[t]
\centering
\caption{Serial SAG-Sep + CVRPSEP hybrid versus the SAG-Sep baseline on $N=50$--$400$. Each cell reports mean root gap (\%) with mean runtime in seconds in parentheses.}
\label{tab:hybrid_cvrpsep_appendix}
\scriptsize
\setlength{\tabcolsep}{4pt}
\renewcommand{\arraystretch}{1.05}
\begin{adjustbox}{max width=\linewidth}
\begin{tabular}{lcccccc}
\toprule
Method & $N=50$ & $N=75$ & $N=100$ & $N=200$ & $N=300$ & $N=400$ \\
\midrule
SAG-Sep baseline & 4.37(0.30) & 5.91(1.20) & 7.11(2.80) & 9.10(23.00) & 9.42(343.30) & 9.98(1897.30) \\
SAG-Sep + CVRPSEP & 1.97(0.40) & 2.78(1.70) & 4.55(4.50) & 6.29(62.30) & 7.24(781.50) & 8.69(2719.10) \\
\bottomrule
\end{tabular}
\end{adjustbox}
\end{table}

\subsection{Detailed $K$-hop Ablation}
\label{sec:hop_ablation}

Table~\ref{tab:hop-ablation-compact} reports the ablation over the $K$-hop
closure radius. Larger hop radii occasionally help on the smallest instances
but scale poorly: from $N=100$ onward, $K=2$ gives the best gap at every size,
and the wider radii are consistently weaker. We therefore use $K=2$ as the
default for all models.

\begin{table*}[t]
\centering
\caption{Ablation on the $K$-hop closure parameter ($N\le 500$). We report mean lower bound (LB), mean runtime (Time, seconds), and mean optimality gap (Gap, \%). Best Gap per $N$ is in bold. Separately trained models from an earlier development stage; absolute gaps differ from Table~\ref{tab:cvrp_3panel_iters_with_exact}, and comparisons are within the table.}
\label{tab:hop-ablation-compact}
\scriptsize
\setlength{\tabcolsep}{4pt}
\renewcommand{\arraystretch}{1.05}
\begin{adjustbox}{max width=\textwidth}
\begin{tabular}{r rrr rrr rrr rrr}
\toprule
\multirow{2}{*}{N}
& \multicolumn{3}{c}{2-hop}
& \multicolumn{3}{c}{3-hop}
& \multicolumn{3}{c}{4-hop}
& \multicolumn{3}{c}{8-hop} \\
\cmidrule(lr){2-4}\cmidrule(lr){5-7}\cmidrule(lr){8-10}\cmidrule(lr){11-13}
& LB & Time & Gap
& LB & Time & Gap
& LB & Time & Gap
& LB & Time & Gap \\
\midrule
 50  & 8958.65 & 0.64 & 6.23\%
     & 8951.04 & 0.70 & 6.37\%
     & 8955.10 & 0.85 & 6.15\%
     & 8995.99 & 0.79 & \textbf{5.61\%} \\
 75  & 12864.96 & 1.94 & 7.00\%
     & 12926.47 & 2.19 & 6.41\%
     & 12890.30 & 2.08 & 6.67\%
     & 12935.72 & 2.72 & \textbf{6.35\%} \\
100  & 15583.42 & 4.98 & \textbf{6.82\%}
     & 15498.78 & 4.24 & 7.38\%
     & 15484.65 & 5.26 & 7.41\%
     & 15456.56 & 5.79 & 7.56\% \\
200  & 20816.81 & 36.29 & \textbf{8.83\%}
     & 20770.05 & 32.81 & 9.00\%
     & 20672.88 & 40.99 & 9.48\%
     & 20683.41 & 48.17 & 9.37\% \\
300  & 31424.26 & 483.62 & \textbf{9.52\%}
     & 31164.89 & 290.98 & 10.04\%
     & 31047.29 & 312.28 & 10.49\%
     & 31182.41 & 459.56 & 10.15\% \\
400  & 45505.74 & 2066.06 & \textbf{9.82\%}
     & 44908.79 & 1446.56 & 10.88\%
     & 44499.48 & 1322.68 & 11.66\%
     & 44940.15 & 2103.80 & 10.92\% \\
500  & 53800.54 & 3072.64 & \textbf{11.12\%}
     & 53048.32 & 2768.78 & 12.16\%
     & 52523.67 & 2459.50 & 13.28\%
     & 52643.90 & 3203.49 & 12.81\% \\
\bottomrule
\end{tabular}
\end{adjustbox}
\end{table*}

\subsection{Paired Per-Instance Comparisons}
\label{sec:paired_appendix}

Table~\ref{tab:paired_main} complements the mean gaps of
Tables~\ref{tab:cvrp_3panel_iters_with_exact} and \ref{tab:test_time_refinement}
with paired per-instance statistics on the ten benchmark instances per size:
win/loss/tie counts, the mean gap difference with its standard error, and the
worst-case difference, for SAG-Sep+OS against the threshold readout, against
Onion-MBP, and against $\pi$-NeuralSEP$_2$+GC. The $\pi$-NeuralSEP$_2$+GC runs
used for the paired comparison were re-run after the OS runs on the same
hardware configuration (one L4 GPU, four CPU cores), so their mean gaps (3.84 /
4.58 / 5.89 / 7.76 / 8.73 / 10.77 / 12.91 / 20.29 / 26.04 for $N=50$ to $1000$)
differ slightly from the values in Table~\ref{tab:cvrp_3panel_iters_with_exact};
the threshold runs are those of Table~\ref{tab:test_time_refinement}. OS
improves on the threshold readout on 85 of 90 instances; on Onion-MBP on 71 of
90, losing on average only at $N=500$ (where a single instance accounts for the
$-2.26$ pp worst case) and tying on average at $N=1000$ ($+0.01$ pp, 3 of 10
wins); and on $\pi$-NeuralSEP$_2$+GC on 77 of 90, winning all 30 instances at
$N\ge 500$ with positive worst-case margins.

\begin{table}[t]
\centering
\caption{Paired per-instance comparison of SAG-Sep+OS on the ten benchmark instances per size. W/L/T counts instances on which OS attains a better/worse/equal root bound; $\Delta$ is the opponent's gap minus the OS gap in percentage points (positive favors OS), reported as mean $\pm$ standard error and as the worst case over the ten instances.}
\label{tab:paired_main}
\scriptsize
\setlength{\tabcolsep}{3pt}
\renewcommand{\arraystretch}{1.05}
\begin{adjustbox}{max width=\linewidth}
\begin{tabular}{r|rrr|rrr|rrr}
\toprule
& \multicolumn{3}{c|}{vs.\ SAG-Sep (threshold)} & \multicolumn{3}{c|}{vs.\ SAG-Sep+Onion-MBP} & \multicolumn{3}{c}{vs.\ $\pi$-NeuralSEP$_2$+GC} \\
$N$ & W/L/T & mean $\Delta\pm$SE & worst $\Delta$ & W/L/T & mean $\Delta\pm$SE & worst $\Delta$ & W/L/T & mean $\Delta\pm$SE & worst $\Delta$ \\
\midrule
50   & 10/0/0 & $+1.77\pm0.29$ & $+0.27$ & 9/1/0  & $+0.87\pm0.19$ & $-0.01$ & 10/0/0 & $+1.24\pm0.20$ & $+0.36$ \\
75   & 10/0/0 & $+2.11\pm0.51$ & $+0.65$ & 8/2/0  & $+0.37\pm0.11$ & $-0.14$ & 9/1/0  & $+0.77\pm0.25$ & $-0.71$ \\
100  & 10/0/0 & $+1.68\pm0.28$ & $+0.44$ & 10/0/0 & $+0.46\pm0.09$ & $+0.11$ & 8/2/0  & $+0.45\pm0.19$ & $-0.56$ \\
200  & 10/0/0 & $+1.49\pm0.15$ & $+0.90$ & 10/0/0 & $+0.40\pm0.09$ & $+0.08$ & 5/5/0  & $+0.14\pm0.14$ & $-0.41$ \\
300  & 10/0/0 & $+0.86\pm0.13$ & $+0.47$ & 9/1/0  & $+0.29\pm0.06$ & $-0.05$ & 6/4/0  & $+0.17\pm0.21$ & $-0.76$ \\
400  & 10/0/0 & $+0.79\pm0.09$ & $+0.34$ & 9/1/0  & $+0.21\pm0.09$ & $-0.44$ & 9/1/0  & $+1.58\pm0.31$ & $-0.23$ \\
500  & 8/2/0  & $+0.57\pm0.22$ & $-0.66$ & 6/4/0  & $-0.12\pm0.28$ & $-2.26$ & 10/0/0 & $+2.12\pm0.37$ & $+0.43$ \\
750  & 8/2/0  & $+0.39\pm0.14$ & $-0.32$ & 7/3/0  & $+0.22\pm0.14$ & $-0.36$ & 10/0/0 & $+3.27\pm0.34$ & $+2.14$ \\
1000 & 9/1/0  & $+0.78\pm0.19$ & $-0.00$ & 3/7/0  & $+0.01\pm0.17$ & $-0.85$ & 10/0/0 & $+4.79\pm0.22$ & $+3.80$ \\
\midrule
all  & 85/5/0 & & & 71/19/0 & & & 77/13/0 & & \\
\bottomrule
\end{tabular}
\end{adjustbox}
\end{table}

\subsection{Generalization Beyond the Benchmark}
\label{sec:generalization_appendix}

To test transfer beyond the NeuralSEP benchmark distribution, we evaluate the
frozen checkpoint, trained at $N=100$, on two XML parameter settings at unseen
sizes $N=500$, $750$, and $1000$ (ten instances per size, 5--10$\times$ the
training size) and on ten CVRPLIB X instances with $N$ between $524$ and $979$,
all under the same 3600-second root limit.
Table~\ref{tab:xml_generalization} reports the XML regimes and
Table~\ref{tab:xset_generalization} the X instances; all gaps are computed
against the best known solutions; the XML instances and their best known
solutions are included in the released repository (Appendix~\ref{sec:code_release}). OS wins all 60 paired comparisons against
$\pi$-NeuralSEP$_2$+GC on the XML regimes and 9 of 10 on the X instances
(losing X-n524-k153 by $0.36$ pp), and attains a lower gap than CVRPSEP on all
70 instances. On X-n524-k153 the demand per customer is unusually large
relative to the vehicle capacity (153 routes for 524 customers), so the useful
cuts are many small, disjoint subsets that the prefix-based decoder covers
only partially; the test-time cut-refinement layer of
Appendix~\ref{sec:code_release} compensates for this (gap $17.04\%$ on this
instance, against $19.40\%$ for OS and $19.04\%$ for $\pi$-NeuralSEP$_2$+GC),
so we did not pursue it further. The $\pi$-NeuralSEP$_2$+GC and CVRPSEP baselines are run from the
NeuralSEP code base (CVRPSEP with its RCC routines only), with the same
single-thread LP setting as in Table~\ref{tab:cvrp_3panel_iters_with_exact}.

\begin{table}[t]
\centering
\caption{Generalization to two XML parameter settings at unseen sizes, 5--10$\times$ the training size (ten instances per cell, 3600-second limit). Mean root gap (\%); W/T/L counts the paired per-instance comparisons of OS against each baseline. XML-A: central depot, random customers, many small and few large demands, medium routes. XML-B: corner depot, random customers, small high-variance demands, medium routes.}
\label{tab:xml_generalization}
\scriptsize
\setlength{\tabcolsep}{5pt}
\renewcommand{\arraystretch}{1.05}
\begin{tabular}{ll|ccc|cc}
\toprule
Regime & $N$ & OS & $\pi$-NeuralSEP$_2$+GC & CVRPSEP & OS vs.\ $\pi$-NeuralSEP$_2$+GC & OS vs.\ CVRPSEP \\
\midrule
XML-A & 500  & \textbf{8.37}  & 8.85  & 11.54 & 10/0/0 & 10/0/0 \\
XML-A & 750  & \textbf{13.47} & 15.47 & 18.25 & 10/0/0 & 10/0/0 \\
XML-A & 1000 & \textbf{17.35} & 20.77 & 22.35 & 10/0/0 & 10/0/0 \\
XML-B & 500  & \textbf{8.86}  & 11.78 & 16.53 & 10/0/0 & 10/0/0 \\
XML-B & 750  & \textbf{17.19} & 19.56 & 23.07 & 10/0/0 & 10/0/0 \\
XML-B & 1000 & \textbf{21.61} & 25.50 & 27.27 & 10/0/0 & 10/0/0 \\
\bottomrule
\end{tabular}
\end{table}

\begin{table}[t]
\centering
\caption{Root gap (\%) on ten CVRPLIB X instances drawn uniformly from $N\in[500,1001]$ (3600-second limit).}
\label{tab:xset_generalization}
\scriptsize
\setlength{\tabcolsep}{5pt}
\renewcommand{\arraystretch}{1.05}
\begin{tabular}{l|ccc}
\toprule
Instance & OS & $\pi$-NeuralSEP$_2$+GC & CVRPSEP \\
\midrule
X-n524-k153 & 19.40 & \textbf{19.04} & 23.49 \\
X-n573-k30  & \textbf{3.36}  & 5.14  & 5.81  \\
X-n613-k62  & \textbf{11.42} & 13.57 & 16.36 \\
X-n701-k44  & \textbf{12.54} & 16.24 & 19.49 \\
X-n716-k35  & \textbf{16.94} & 19.65 & 20.75 \\
X-n749-k98  & \textbf{19.14} & 20.49 & 26.00 \\
X-n766-k71  & \textbf{13.97} & 17.28 & 19.81 \\
X-n837-k142 & \textbf{20.37} & 23.96 & 26.42 \\
X-n936-k151 & \textbf{20.67} & 25.85 & 24.42 \\
X-n979-k58  & \textbf{23.34} & 28.23 & 27.06 \\
\midrule
Mean        & \textbf{16.11} & 18.94 & 20.96 \\
\bottomrule
\end{tabular}
\end{table}

\section{Code Release and Scope Beyond the Root Node}
\label{sec:code_release}

The code accompanying this paper is released at
\url{https://github.com/henryleo180/SAG-Sep}. The repository contains a C11
implementation of the separator (ONNX inference, root-only solver and a
library interface for host solvers), a Python reference implementation, the
trained checkpoint and its ONNX export, the training pipeline, and the
benchmark instances used in this paper. The exact-label set the checkpoint was
trained on is archived at \url{https://doi.org/10.5281/zenodo.23539579}. The \texttt{os} recipe reproduces the
SAG-Sep+OS row of Table~\ref{tab:test_time_refinement}; the C and Python implementations decode the
same candidates and differ only in floating-point details.

The repository also goes beyond the scope of the present paper. It ships branch-and-cut examples
in which SAG-Sep separates RCCs for the capacitated, asymmetric, open and
time-windowed VRP and for the capacitated minimum spanning tree problem, a
branch-price-and-cut integration on top of RouteOpt~\citep{you2025routeopt}
for the CVRP and the VRPTW, and a training-free test-time cut-refinement layer (TTS)
for the rounding-driven states that dominate beyond the root node
(Section~\ref{sec:limitations}); the branch-and-cut examples use the paper's
checkpoint, and the branch-price-and-cut integration ships a second checkpoint
fine-tuned from it on column-generation states. These
components are not part of the contributions claimed here and are evaluated
only preliminarily below. We report these preliminary results mainly to address
questions raised during the review about the framework's potential inside
branch-and-cut and branch-price-and-cut solvers; the main text still focuses
only on the root node, which is also the setting in which
NeuralSEP$_1$~\citep{kim2024neural} and NeuralSEP$_2$~\citep{sim2025test} are
developed and evaluated. A forthcoming extended version describes and evaluates
these components in detail.

\subsection{The Released Engine Run to Closure Against CVRPSEP}
\label{sec:engine_vs_cvrpsep}

The main tables stop at a fixed root time limit and compare separators by the
bound they reach within it. Table~\ref{tab:engine_closure} instead reports the
released C engine with its full recipe (OS followed by TTS) and the CVRPSEP separator, both run in the same root cutting-plane
loop to closure, that is, until the separator proposes no admissible cut, with
tail-off stopping disabled, on the ten random benchmark instances per size.
Both reach the same root bound at every size to within $0.002$ points of gap; the learned separator runs its
forward pass on the GPU while CVRPSEP runs on the CPU, so the time columns are
a systems comparison rather than a like-for-like separator speed. Under this
protocol the learned engine is slower than CVRPSEP at $N=50$ and $N=75$,
and faster from $N=100$ on, by a factor of three at $N=200$. On a single CPU
core (no GPU) the engine needs 0.50, 6.9 and 67 seconds per instance at
$N=50$, $100$ and $200$, with the same bounds to within $0.002$ points, so the
crossover moves to $N=200$.

\begin{table}[t]
\centering
\caption{Released C engine (full recipe) versus CVRPSEP, both run to closure in the same root loop (ten random instances per size, means). Gap is with respect to the optimal solution. The time ratio is engine time over CVRPSEP time, computed from unrounded means.}
\label{tab:engine_closure}
\scriptsize
\setlength{\tabcolsep}{5pt}
\renewcommand{\arraystretch}{1.05}
\begin{tabular}{l|cc|cc|c}
\toprule
$N$ & CVRPSEP gap (\%) & Engine gap (\%) & CVRPSEP s/inst (CPU) & Engine s/inst (GPU) & Time ratio \\
\midrule
50  & 1.969 & 1.968 & 0.09 & 0.16 & 1.86$\times$ \\
75  & 2.768 & 2.768 & 0.63 & 0.74 & 1.19$\times$ \\
100 & 4.539 & 4.538 & 2.93 & 2.20 & 0.75$\times$ \\
200 & 6.281 & 6.283 & 95.8 & 32.1 & 0.34$\times$ \\
\bottomrule
\end{tabular}
\end{table}

\subsection{Preliminary Complete Branch-and-Cut}
\label{sec:bc_fulltree}

Table~\ref{tab:bc_fulltree} reports complete branch-and-cut runs, with a
two-hour limit per run, on the five $N=75$ and $N=100$ benchmark instances
that CVRPSEP solves within that limit, so the table compares solve times on
these instances rather than solve rates. Four separation arms are compared in
the same research harness: OS followed by CVRPSEP once OS stalls; OS with the
test-time cut-refinement layer and no call to CVRPSEP; $\pi$-NeuralSEP$_2$+GC
followed by CVRPSEP, run in its own code base; and CVRPSEP alone. On every
instance the four root bounds agree to within $0.06$ percentage points, so the
differences in time and tree size come mainly from the tree search. The refinement layer carries all five complete
trees without any CVRPSEP call, and on random-18 it is the fastest
configuration with the smallest tree; CVRPSEP is the fastest on random-13 and
random-37, and OS+CVRPSEP on random-21. The released examples reproduce the CVRPSEP and
refinement-layer arms; the other two arms come from the research harness.

\begin{table}[t]
\centering
\caption{Preliminary complete branch-and-cut on the five $N{=}75/100$ benchmark instances that CVRPSEP solves within the two-hour limit (two-index formulation with lazy RCC callbacks, 2-hour limit, one LP thread, cutoff at the optimum plus one, so the runs measure the time to prove optimality, user cuts at the first 31 nodes, integer solutions checked exactly; one node type with a B200 GPU, not the L4 machine of the main tables). Each cell gives wall-clock seconds / tree nodes; TL marks the time limit, with the remaining gap to the optimum (\%) in parentheses. OS+CVRPSEP: OS at the root until it stalls, then CVRPSEP; both separate in the tree. TTS: OS with a training-free test-time cut-refinement layer (a research version of the layer in the released code) at the root and in the tree, without any call to CVRPSEP. $\pi$-NeuralSEP$_2$+GC+CVRPSEP runs in its own code base and is included as a reference; the CVRPSEP column is an earlier run of the same harness.}
\label{tab:bc_fulltree}
\scriptsize
\setlength{\tabcolsep}{4pt}
\renewcommand{\arraystretch}{1.05}
\begin{adjustbox}{max width=\linewidth}
\begin{tabular}{l|cccc}
\toprule
Instance & OS+CVRPSEP & TTS (no CVRPSEP) & $\pi$-NeuralSEP$_2$+GC+CVRPSEP & CVRPSEP \\
\midrule
random-13-X-n75  & 5 / 376              & 12 / 425                  & 34 / 1{,}515          & \textbf{4} / 226 \\
random-21-X-n75  & \textbf{22} / 2{,}832 & 32 / 4{,}134             & 83 / 5{,}134          & 35 / 3{,}804 \\
random-37-X-n75  & 78 / 13{,}319        & 113 / 19{,}323            & 179 / 8{,}725         & \textbf{53} / 8{,}427 \\
random-18-X-n100 & 346 / 31{,}215       & \textbf{122} / 11{,}046   & 609 / 44{,}518        & 1185 / 114{,}243 \\
random-34-X-n100 & 2255 / 208{,}434     & \textbf{1980} / 180{,}459 & TL (0.51)             & 2100 / 186{,}129 \\
\bottomrule
\end{tabular}
\end{adjustbox}
\end{table}

\subsection{Capacity and Rounding Parts of Exact RCC Violations}
\label{sec:rcc_decomposition}

The violation of an RCC splits into a capacity part, $2d(S)/C-\bar x(\delta(S))$,
and a rounding part, $2(\lceil d(S)/C\rceil-d(S)/C)$. Table~\ref{tab:rcc_decomposition}
classifies the 25{,}411 exact cuts found in complete branch-and-cut runs with
exact separation on 30 benchmark instances ($N=50$, $75$, $100$) by the phase
in which they were found. Early root cuts are overwhelmingly capacity-dominated
and large; by the second half of the root loop, and at every tree node, the
capacity part is non-positive for essentially all cuts and the remaining
violations are small and rounding-driven. The transition therefore happens
within the root loop through repeated separation; branching adds nothing
qualitatively new, which is why separators tuned to the early, batch-oriented
root regime need a different mechanism beyond it (Section~\ref{sec:limitations}).

\begin{table}[t]
\centering
\caption{Capacity and rounding parts of the violation of exact RCCs in complete branch-and-cut runs with exact separation (30 random benchmark instances, $N=50$, $75$, $100$; 25{,}411 cuts). A cut is capacity-dominated if its capacity part exceeds its rounding part; violations are in the units of $\bar x(\delta(S))$. Root round $r$ of an instance's $R$ rounds falls in the first tenth if $r/R\le 0.1$ and in the first half if $r/R\le 0.5$; none of the 6{,}424 late-root and 8{,}453 tree cuts is capacity-dominated.}
\label{tab:rcc_decomposition}
\scriptsize
\setlength{\tabcolsep}{5pt}
\renewcommand{\arraystretch}{1.05}
\begin{tabular}{l|rrrr}
\toprule
Phase & Cuts & Mean violation & Capacity-dominated (\%) & Capacity part $>0$ (\%) \\
\midrule
Root, first 10\% of rounds & 2{,}032 & 7.88 & 86.1 & 96.9 \\
Root, 10--50\% of rounds   & 8{,}502 & 2.34 & 22.5 & 48.8 \\
Root, last 50\% of rounds  & 6{,}424 & 0.24 & 0.0  & 0.5 \\
Tree nodes                 & 8{,}453 & 0.13 & 0.0  & 0.2 \\
\bottomrule
\end{tabular}
\end{table}

\subsection{Transfer Across Problem Classes: The Released Examples}
\label{sec:variant_bc}

The released branch-and-cut examples run the engine zero-shot, with the
paper's checkpoint and its full recipe (OS followed by TTS), on four VRP variants and on the
capacitated minimum spanning tree problem, against CVRPSEP in the same driver
(Tables~\ref{tab:variant_bc_vrp} and \ref{tab:variant_bc_cmst}; one CPU thread,
1800-second limit, one run per instance). On the capacitated, asymmetric and
time-windowed variants (30 instances) the two separators reach identical root
bounds and both solve all 30; on the hardest VRPTW instance (rc101\_100) the
engine finishes in 478 seconds against 1785 for CVRPSEP. On the open VRP the
root bounds agree on 3 of 7 instances and differ by at most 0.004 points of gap
on the other 4; each separator solves the same 3, and on the 4 unsolved the
engine ends with the higher bound on 3. The engine reaches these root bounds
with 50--76\% of the root cuts CVRPSEP adds, but its CPU root loop is slower,
so total times split evenly: among the runs that both separators solve and in
which the slower one takes more than one second, it is faster on 6, slower on 6
and level on 1 (differences within 10\% counted as level). On the CMST,
a different problem class with the capacity cut as its only family and no CMST
training, the root bounds differ on 19 of 20 instances: the engine is higher on
13, by up to 5.8 points of gap at tight capacity ($Q\in\{200,400\}$), and lower
on 6 by at most 0.44 points; it solves 3 instances against 1 within the limit,
a count that changes with the tree-commitment setting of the driver. These are
runs of the released examples rather than a benchmark: one seed, small
instance sets, and timings from shared cluster nodes.

\begin{table}[t]
\centering
\caption{Released branch-and-cut examples on four VRP variants: CVRPSEP versus SAG-Sep, one run per instance. Both arms use the same C driver, formulation, fleet ($k$ routes, fixed by an equality), root and tree schedule, exact rescoring and admission rule, and differ only in the source of rounded capacity cuts: the root loop runs to closure (at most 100 rounds of at most 100 cuts), user cuts are separated at the first 15 tree nodes, and every integer solution is checked exactly; one CPU thread, 1800-second limit, relative MIP gap $10^{-4}$, the MIP solver's own cuts and heuristics at their defaults. SAG-Sep is the released C engine with its full recipe, including the test-time cut-refinement layer, and the paper's checkpoint, applied zero-shot through each example's directed or split-depot aggregation, with its forward pass on the CPU. Gap: root bound relative to the reference value (\%); Root and Total: seconds; Nodes: tree nodes; TL: time limit, with the remaining gap of the final lower bound (\%) in parentheses. References are proved optima except those marked $^\dagger$ (best known). Counts in the text use unrounded bounds. Instances: four $N{=}50$ random benchmark instances of this paper (CVRP); the Fischetti--Toth--Vigo asymmetric instances (ACVRP); seven Christofides--Mingozzi--Toth instances as open VRPs (OVRP); Solomon series 1 and 2 with 25, 50 and 100 customers, distances truncated to one decimal (VRPTW).}
\label{tab:variant_bc_vrp}
\scriptsize
\setlength{\tabcolsep}{3pt}
\renewcommand{\arraystretch}{1.0}
\begin{adjustbox}{max width=\linewidth}
\begin{tabular}{lrrr|rrrr|rrrr}
\toprule
 & & & & \multicolumn{4}{c|}{CVRPSEP} & \multicolumn{4}{c}{SAG-Sep (released engine, OS+TTS)} \\
Instance & $n$ & $k$ & Ref. & Gap & Root & Total & Nodes & Gap & Root & Total & Nodes \\
\midrule
\multicolumn{12}{l}{\textit{CVRP}} \\
random-4-X-n50 & 50 & 4 & 7684 & 1.37 & 0.02 & 1.4 & 508 & 1.37 & 0.23 & 4.9 & 1{,}080 \\
random-8-X-n50 & 50 & 3 & 6733 & 0.90 & 0.02 & 0.4 & 137 & 0.90 & 0.18 & 0.8 & 111 \\
random-12-X-n50 & 50 & 4 & 8317 & 0.89 & 0.06 & 13.8 & 603 & 0.89 & 0.26 & 4.7 & 345 \\
random-24-X-n50 & 50 & 3 & 6686 & 2.15 & 0.02 & 7.0 & 1{,}494 & 2.15 & 0.20 & 3.2 & 2{,}046 \\
\midrule
\multicolumn{12}{l}{\textit{ACVRP}} \\
A034-02f & 33 & 2 & 1406 & 2.56 & 0.01 & 0.7 & 40 & 2.56 & 0.05 & 1.2 & 29 \\
A036-03f & 35 & 3 & 1644 & 0.74 & 0.03 & 0.3 & 16 & 0.74 & 0.12 & 0.6 & 20 \\
A039-03f & 38 & 3 & 1654 & 0.00 & 0.01 & 0.0 & 1 & 0.00 & 0.08 & 0.1 & 1 \\
A045-03f & 44 & 3 & 1740 & 0.00 & 0.04 & 0.3 & 1 & 0.00 & 0.15 & 0.4 & 1 \\
A048-03f & 47 & 3 & 1891 & 0.32 & 0.02 & 0.2 & 1 & 0.32 & 0.11 & 0.5 & 1 \\
A056-03f & 55 & 3 & 1739 & 1.24 & 0.02 & 0.7 & 87 & 1.24 & 0.16 & 1.8 & 20 \\
A065-03f & 64 & 3 & 1974 & 1.41 & 0.08 & 3.7 & 1 & 1.41 & 0.19 & 2.6 & 1 \\
A071-03f & 70 & 3 & 2054 & 1.36 & 0.02 & 0.1 & 1 & 1.36 & 0.18 & 0.4 & 1 \\
\midrule
\multicolumn{12}{l}{\textit{OVRP}} \\
O1 & 50 & 5 & 416.063 & 1.72 & 0.02 & 16.0 & 2{,}128 & 1.72 & 0.18 & 11.5 & 1{,}069 \\
O2 & 75 & 10 & 567.14$^\dagger$ & 3.83 & 0.15 & TL (2.35) & 45{,}007 & 3.83 & 0.90 & TL (2.41) & 44{,}136 \\
O3 & 100 & 8 & 639.744 & 1.22 & 0.18 & 313.3 & 29{,}188 & 1.22 & 0.64 & 123.2 & 11{,}756 \\
O4 & 150 & 12 & 733.13$^\dagger$ & 1.36 & 1.30 & TL (0.69) & 34{,}448 & 1.37 & 3.49 & TL (0.66) & 28{,}912 \\
O5 & 199 & 16 & 896.08$^\dagger$ & 6.03 & 9.97 & TL (5.73) & 2{,}563 & 6.03 & 14.69 & TL (5.61) & 13{,}293 \\
O6 & 120 & 7 & 682.12$^\dagger$ & 5.77 & 3.27 & TL (4.74) & 15{,}123 & 5.77 & 4.15 & TL (4.32) & 21{,}298 \\
O7 & 100 & 10 & 534.239 & 1.46 & 0.79 & 72.5 & 986 & 1.46 & 2.48 & 69.4 & 1{,}644 \\
\midrule
\multicolumn{12}{l}{\textit{VRPTW}} \\
c101\_25 & 25 & 3 & 191.3 & 0.00 & 0.00 & 0.0 & 1 & 0.00 & 0.01 & 0.0 & 1 \\
c101\_50 & 50 & 5 & 362.4 & 0.00 & 0.00 & 0.0 & 1 & 0.00 & 0.04 & 0.1 & 1 \\
c101\_100 & 100 & 10 & 827.3 & 0.00 & 0.02 & 0.1 & 1 & 0.00 & 0.26 & 0.3 & 1 \\
r101\_25 & 25 & 8 & 617.1 & 0.00 & 0.00 & 0.0 & 1 & 0.00 & 0.01 & 0.0 & 1 \\
r101\_50 & 50 & 12 & 1044 & 0.93 & 0.01 & 0.0 & 1 & 0.93 & 0.02 & 0.0 & 1 \\
r101\_100 & 100 & 20 & 1637.7 & 1.31 & 0.01 & 0.1 & 1 & 1.31 & 0.09 & 0.8 & 1 \\
rc101\_25 & 25 & 4 & 461.1 & 11.14 & 0.00 & 0.1 & 171 & 11.14 & 0.02 & 0.2 & 171 \\
rc101\_50 & 50 & 8 & 944 & 13.48 & 0.01 & 5.6 & 3{,}867 & 13.48 & 0.08 & 9.8 & 8{,}954 \\
rc101\_100 & 100 & 15 & 1619.8 & 15.83 & 0.08 & 1784.6 & 399{,}365 & 15.83 & 0.78 & 478.4 & 109{,}372 \\
c201\_25 & 25 & 2 & 214.7 & 0.00 & 0.00 & 0.0 & 1 & 0.00 & 0.01 & 0.0 & 1 \\
c201\_50 & 50 & 3 & 360.2 & 0.00 & 0.00 & 0.0 & 1 & 0.00 & 0.01 & 0.0 & 1 \\
c201\_100 & 100 & 3 & 589.1 & 0.00 & 0.01 & 0.0 & 1 & 0.00 & 0.03 & 0.0 & 1 \\
r201\_25 & 25 & 4 & 463.3 & 2.84 & 0.00 & 0.0 & 1 & 2.84 & 0.01 & 0.0 & 1 \\
r201\_50 & 50 & 6 & 791.9 & 2.27 & 0.01 & 0.1 & 1 & 2.27 & 0.02 & 0.1 & 1 \\
r201\_100 & 100 & 8 & 1143.2 & 7.15 & 0.05 & 12.4 & 1{,}272 & 7.15 & 0.08 & 14.6 & 958 \\
rc201\_25 & 25 & 3 & 360.2 & 0.39 & 0.00 & 0.0 & 1 & 0.39 & 0.01 & 0.0 & 1 \\
rc201\_50 & 50 & 5 & 684.8 & 6.32 & 0.01 & 0.1 & 66 & 6.32 & 0.03 & 0.2 & 42 \\
rc201\_100 & 100 & 9 & 1261.8 & 9.30 & 0.05 & 13.6 & 1{,}401 & 9.30 & 0.12 & 15.6 & 1{,}703 \\
\bottomrule
\end{tabular}
\end{adjustbox}
\end{table}

\begin{table}[t]
\centering
\caption{Released branch-and-cut example on the capacitated minimum spanning tree problem (OR-Library \texttt{capmst} instances; $n$ terminals, branch capacity $Q$; all references are proved optima). The model is the natural undirected edge formulation with the capacity cut as its only dynamic family. Its connectivity members ($\lceil d(S)/Q\rceil=1$) are separated exactly by maximum flow in both arms; sets with $\lceil d(S)/Q\rceil\ge 2$ come from CVRPSEP's CVRP heuristics or from SAG-Sep zero-shot (the engine, recipe and checkpoint of Table~\ref{tab:variant_bc_vrp}, no CMST training), are rescored exactly against the CMST point, and enter in inner and/or cut-set form. Driver settings, columns and limits are those of Table~\ref{tab:variant_bc_vrp}.}
\label{tab:variant_bc_cmst}
\scriptsize
\setlength{\tabcolsep}{3pt}
\renewcommand{\arraystretch}{1.0}
\begin{adjustbox}{max width=\linewidth}
\begin{tabular}{lrrr|rrrr|rrrr}
\toprule
 & & & & \multicolumn{4}{c|}{CVRPSEP} & \multicolumn{4}{c}{SAG-Sep (released engine, OS+TTS)} \\
Instance & $n$ & $Q$ & Ref. & Gap & Root & Total & Nodes & Gap & Root & Total & Nodes \\
\midrule
cm50r1 & 49 & 200 & 1098 & 11.14 & 0.16 & TL (9.47) & 229{,}865 & 8.65 & 1.88 & TL (6.47) & 302{,}447 \\
cm50r2 & 49 & 200 & 974 & 16.31 & 0.15 & TL (13.76) & 262{,}483 & 11.84 & 1.72 & TL (8.83) & 335{,}159 \\
cm50r3 & 49 & 200 & 1186 & 11.52 & 0.25 & TL (10.20) & 194{,}989 & 8.03 & 2.03 & TL (6.07) & 220{,}994 \\
cm50r4 & 49 & 200 & 800 & 11.50 & 0.06 & TL (6.50) & 272{,}333 & 9.99 & 0.96 & TL (5.25) & 282{,}248 \\
cm50r5 & 49 & 200 & 928 & 15.16 & 0.11 & TL (11.31) & 187{,}211 & 10.76 & 2.48 & TL (8.19) & 297{,}886 \\
cm50r1 & 49 & 400 & 679 & 12.02 & 0.08 & TL (10.31) & 160{,}868 & 7.69 & 0.66 & TL (4.86) & 266{,}054 \\
cm50r2 & 49 & 400 & 631 & 13.54 & 0.09 & TL (11.09) & 156{,}504 & 10.32 & 0.73 & TL (6.66) & 307{,}304 \\
cm50r3 & 49 & 400 & 732 & 13.19 & 0.05 & TL (11.07) & 191{,}722 & 7.39 & 0.68 & TL (4.78) & 407{,}596 \\
cm50r4 & 49 & 400 & 564 & 8.52 & 0.07 & TL (3.72) & 246{,}341 & 7.84 & 0.41 & TL (3.01) & 217{,}671 \\
cm50r5 & 49 & 400 & 611 & 10.49 & 0.08 & TL (8.35) & 114{,}945 & 7.96 & 0.71 & TL (4.91) & 315{,}654 \\
cm50r1 & 49 & 800 & 495 & 5.00 & 0.04 & TL (1.62) & 209{,}968 & 4.79 & 0.16 & 899.7 & 119{,}499 \\
cm50r2 & 49 & 800 & 513 & 10.77 & 0.03 & TL (7.02) & 179{,}280 & 11.21 & 0.17 & TL (5.65) & 252{,}559 \\
cm50r3 & 49 & 800 & 532 & 6.07 & 0.06 & TL (3.38) & 208{,}358 & 5.64 & 0.28 & TL (2.26) & 194{,}399 \\
cm50r4 & 49 & 800 & 471 & 5.94 & 0.05 & 22.6 & 9{,}494 & 5.94 & 0.24 & 6.1 & 545 \\
cm50r5 & 49 & 800 & 492 & 7.29 & 0.03 & TL (3.25) & 145{,}122 & 6.71 & 0.23 & TL (0.81) & 630{,}510 \\
cm100r1 & 99 & 800 & 182 & 1.05 & 0.93 & TL (0.55) & 48{,}609 & 1.24 & 2.80 & 241.1 & 11{,}071 \\
cm100r2 & 99 & 800 & 179 & 1.74 & 1.35 & TL (0.56) & 45{,}197 & 1.84 & 3.52 & TL (1.12) & 37{,}352 \\
cm100r3 & 99 & 800 & 175 & 1.67 & 1.41 & TL (1.14) & 43{,}060 & 1.71 & 1.96 & TL (1.14) & 47{,}700 \\
cm100r4 & 99 & 800 & 183 & 1.64 & 0.65 & TL (1.09) & 36{,}776 & 1.64 & 2.37 & TL (1.09) & 37{,}617 \\
cm100r5 & 99 & 800 & 186 & 1.88 & 1.67 & TL (1.61) & 34{,}764 & 1.95 & 2.37 & TL (1.08) & 40{,}232 \\
\bottomrule
\end{tabular}
\end{adjustbox}
\end{table}

\end{document}